\documentclass{article}

\usepackage{microtype}
\usepackage{graphicx}
\usepackage{subcaption}
\usepackage{booktabs} 
\usepackage{pifont}
\usepackage{hyperref}

\usepackage[preprint]{./icml_style/icml2026}

\usepackage{amsmath}
\usepackage{amssymb}
\usepackage{mathtools}
\usepackage{amsthm}
\usepackage{enumitem}
\usepackage{multirow}

\usepackage[capitalize,noabbrev]{cleveref}

\theoremstyle{plain}
\newtheorem{theorem}{Theorem}[section]
\newtheorem{proposition}[theorem]{Proposition}
\newtheorem{lemma}[theorem]{Lemma}

\theoremstyle{definition}

\theoremstyle{remark}
\newtheorem{remark}[theorem]{Remark}
\newtheorem{example}[theorem]{Example}

\newcommand{\R}{\mathbb{R}}

\newcommand{\supp}{\mathrm{supp}}
\newcommand{\Id}{\mathrm{Id}}
\newcommand{\cP}{\mathcal{P}}
\newcommand{\cX}{\mathcal{X}}
\newcommand{\cY}{\mathcal{Y}}
\newcommand{\cM}{\mathcal{M}}
\newcommand{\cL}{\mathcal{L}}
\newcommand{\eps}{\varepsilon}

\DeclareMathOperator*{\argmin}{arg\,min}
\DeclareMathOperator*{\argmax}{arg\,max}

\definecolor{darkgreen}{rgb}{0,0.4,0}
\newcommand{\dk}[1]{{\color{purple}[DK: #1]}}

\newcommand{\cmark}{\textcolor{green!60!black}{\ding{51}}}
\newcommand{\xmark}{\textcolor{red!70!black}{\ding{55}}}
\newcommand{\dashmark}{\textcolor{black}{--}}

\usepackage[textsize=tiny]{todonotes}

\icmltitlerunning{
Rate-Optimal Noise Annealing in Semi-Dual Neural Optimal Transport
}

\begin{document}

\twocolumn[

  \icmltitle{
  Rate-Optimal Noise Annealing in Semi-Dual Neural Optimal Transport: Tangential Identifiability, Off-Manifold Ambiguity, and Guaranteed Recovery
  }

  \icmlsetsymbol{equal}{*}

\begin{icmlauthorlist}
  \icmlauthor{Raymond Chu}{cmu}
  \icmlauthor{Jaewoong Choi}{equal,skku}
  \icmlauthor{Dohyun Kwon}{equal,uos,kias}
\end{icmlauthorlist}

 \icmlaffiliation{cmu}{Department of Mathematical Sciences, Carnegie Mellon University, Pittsburgh, PA, USA}
\icmlaffiliation{skku}{Department of Statistics, Sungkyunkwan University, Seoul, Republic of Korea}
\icmlaffiliation{uos}{Department of Mathematics, University of Seoul, Seoul, Republic of Korea}
\icmlaffiliation{kias}{Korea Institute for Advanced Study, Seoul, Republic of Korea}

\icmlcorrespondingauthor{Dohyun Kwon}{dh.dohyun.kwon@gmail.com}
\icmlcorrespondingauthor{Jaewoong Choi}{jaewoongchoi@skku.edu}


  \icmlkeywords{Machine Learning, ICML}

  \vskip 0.3in
]



\printAffiliationsAndNotice{\icmlEqualContribution}

\begin{abstract}
Semi-dual neural optimal transport learns a transport map via a max--min objective, yet training can converge to incorrect or degenerate maps. We fully characterize these spurious solutions in the common regime where data concentrate on low-dimensional manifold: the objective is underconstrained off the data manifold, while the on-manifold transport signal remains identifiable. Following~\citet{choi2025overcoming}, we study additive-noise smoothing as a remedy and prove new map recovery guarantees as the noise vanishes. Our main practical contribution is a computable terminal noise level $\varepsilon_{\mathrm{stat}}(N)$ that attains the optimal statistical rate, with scaling governed by the intrinsic dimension $m$ of the data. The formula arises from a theoretical unified analysis of (i) quantitative stability of optimal plans, (ii) smoothing-induced bias, and (iii) finite-sample error, yielding rates that depend on $m$ rather than the ambient dimension. Finally, we show that the reduced semi-dual objective becomes increasingly ill-conditioned as $\varepsilon \downarrow 0$. This provides a principled stopping rule: annealing below $\varepsilon_{\mathrm{stat}}(N)$ can \emph{worsen} optimization conditioning without improving statistical accuracy.
\end{abstract}

\section{Introduction}
\label{sec:introduction}

Optimal transport (OT) provides a natural framework for comparing probability measures. 
Neural approaches to OT have found broad use across a wide range of applications, including generative modeling  \citep{not, uotm, sjko, lipman2022flow}.
A prominent family of methods  \citep{fanscalable,otm} is based on the \emph{Semi-dual Neural OT \textbf{(SNOT)}}, in which one learns a potential and then recovers an associated transport map via the max--min objective. Given source and target probability measures $\mu$ and $\nu$ and a cost function $c$, these approaches optimize
\begin{equation} 
\label{eq:snot_intro}
\begin{aligned}
    &\sup_{V \in C(\cY)} \inf_{T:\mathcal{X} \rightarrow \mathcal{Y}}
    \mathcal{L}(V, T)
    \quad \hbox{where} \quad \mathcal{L}(V, T):= 
    \\
    & \int_{\mathcal{X}} c(x, T(x)) - V(T(x)) d\mu(x) + \int_{\mathcal{Y}} V(y) d\nu(y) .   
\end{aligned}    
\end{equation}
where $T$ denotes a candidate transport map, $V$ denotes a potential, and $C(\mathcal{Y})$ is a set of continuous functions on $\mathcal{Y}$. If an optimal transport map $T^\star$ exists, it is known that $T^\star$ minimizes the inner objective for some potential $V^\star$, and $(V^\star, T^\star)$ is an optimal solution of \eqref{eq:snot_intro}.

A critical challenge, however, is that \emph{not every optimizer of the semi-dual max--min objective yields a correct transport map}. Even when \eqref{eq:snot_intro} is solved accurately, the recovered $T$ may fail to push $\mu$ to $\nu$, producing an incorrect or degenerate mapping \cite{otm, fanTMLR, choi2025overcoming}. This challenge was referred to as the \emph{spurious solution} problem in \citet{choi2025overcoming}, where examples of various failure cases were investigated.  This pathology is particularly pronounced under the manifold assumption common in high-dimensional inference, where $\mu$ is supported on a lower-dimensional set. 

\paragraph{Why spurious transport maps arise}

To understand this failure mode, we identify a geometric mechanism behind spurious solutions.
In \eqref{eq:snot_intro}, one learns a function $V$ and then \emph{recovers} a transport map $T_V$ by minimizing the semi-dual objective over maps $T$:
\begin{align}
\label{eq:recov}
    T_{V} \in \argmin_{T:\cX\to \cY}\, \mathcal{L}(V,T).
\end{align}
Our first step is to fully characterize this recovery step  (Theorem~\ref{thm:recovery_characterization}). We show that optimal recovery is equivalent to a pointwise condition enforced only for $\mu$-a.e.\ inputs as in \eqref{eq:recovery_constraint}. Under the common high-dimensional assumption that $\mu$ is supported on a low-dimensional manifold, this condition can \textbf{at most identify the tangential component of the optimal transport map} (Theorem~\ref{thm:tangential_recovery}) while its behavior in \textbf{off-manifold (normal) directions remains unconstrained}. This limitation is sharp: in Example~\ref{ex:failure}, we exhibit infinitely many recovered maps that achieve the same objective and match the OT map tangentially, yet have arbitrary normal components.  We observe the same pattern empirically: neural models can attain very small \emph{tangential} error while exhibiting much larger \emph{normal} error (Table~\ref{tab:tang_vs_normal}).

The analysis in Theorem~\ref{thm:recovery_characterization} and Theorem~\ref{thm:tangential_recovery} explains why spurious maps can arise under the manifold assumption, which is common in high-dimensional settings. This also motivates adding \emph{off-manifold regularization} to the source distribution $\mu$ to remove off-support degrees of freedom. Consistent with this expectation, we empirically find that smoothing reduces the normal-direction error (Table~\ref{tab:tang_vs_normal}).

\paragraph{Why smoothing works, and what it guarantees}

A practical way to remove off-manifold ambiguity is to replace the singular source distribution by a full-dimensional regularization $\mu_\varepsilon \ll \mathcal{L}^d$ as in \citet{choi2025overcoming}. 
We prove that smoothing does not discard identifiable structure (Proposition~\ref{thm:noise_tangential}). As $\varepsilon\downarrow 0$, the regularized solutions recover \textbf{the correct \emph{on-manifold} (tangential) transport signal}, matching what is identifiable from the unregularized problem. 

In the most favorable cases, where the limiting potential is differentiable, this can be strengthened to full map recovery, with the learned maps \textbf{converging (in a graphical sense) to an optimal Monge map} as $\varepsilon\downarrow 0$ (Theorem~\ref{thm:map_level_convg}).

\paragraph{Why small noise can be hard to optimize}

Empirically,~\citet{choi2025overcoming} observe that reducing the noise level to a small positive $\varepsilon_{\min}>0$ yields better stability and performance than annealing all the way to $\varepsilon=0$. We show this is driven by \emph{conditioning}. \textbf{In settings where an exact Monge map does not exist} (a common regime for lower-dimensional supports), the \emph{reduced} semi-dual objective (obtained by solving the inner minimization over $T$) can become increasingly \textbf{ill-conditioned as $\varepsilon \to 0$:
the Hessian with respect to network parameters can blow up near the optimizer} (Theorem~\ref{lem:theta_derivatives1}, Lemma~\ref{lem:no_uniform_Lip}). This implies that gradient-based training typically needs smaller steps, more iterations, and becomes more sensitive to hyperparameters at very small noise. Figure~\ref{fig:mse_vs_n} validates this: smaller $\varepsilon$ 
requires more training iterations to converge.

\paragraph{The bias--stability--sample tradeoff}

Noise selection is therefore a tradeoff among three effects:
(i) \emph{stability}, which improves with larger $\varepsilon$ due to better conditioning;
(ii) \emph{bias}, which increases with larger $\varepsilon$ since one solves a smoothed OT problem;
and (iii) \emph{finite-sample estimation error}, since training uses only finitely many samples.
Our analysis makes this tradeoff explicit by bounding the expected error of the optimal plan learned from noisy samples, decomposing it into a bias term and an estimation term (Theorem~\ref{thm:semi_discrete}), when the target measure is sufficiently regular. Under the lower-dimensional support assumption, the estimation term depends on the \textbf{\emph{intrinsic} dimension} rather than the ambient dimension (Figure~\ref{fig:eps_vs_slopes}), leveraging intrinsic Wasserstein estimation phenomena~\cite{block2021intrinsic}.

A key implication of Theorem~\ref{thm:semi_discrete} is that for a fixed sample size $N$, \textbf{accuracy cannot improve indefinitely by taking $\varepsilon \downarrow 0$}.
Below a statistically meaningful scale, the finite-sample estimation term dominates, so further reducing the additive noise $\eps Y$ does not improve the rate. At the same time, smaller $\varepsilon$ can severely deteriorate optimization conditioning (Theorem~\ref{lem:theta_derivatives1}, Lemma~\ref{lem:no_uniform_Lip}). 

This motivates a \textbf{terminal noise level} $\varepsilon_{\mathrm{stat}}(N)$: the \emph{largest} $\varepsilon$
(up to scaling constants) that still attains the statistical optimal rate; taking $\varepsilon \ll \varepsilon_{\mathrm{stat}}(N)$
cannot improve statistical rate and may only worsen conditioning (see Figure~\ref{fig:neural_map_fit}). Balancing bias and estimation in
Theorem~\ref{thm:semi_discrete} yields
\begin{align}
\label{eq:estat}    \varepsilon_{\mathrm{stat}}(N)
\propto \frac{1}{\mathbb{E}[|Y|]}
\begin{cases}
N^{-1/2} & m=1,\\
(\log N/N)^{1/2} & m=2,\\
N^{-1/m} & m\ge 3.
\end{cases}
\end{align}
Here $m$ is the intrinsic dimension of the support of $\mu$.

When one uses Gaussian noise i.e. $Y\sim\mathcal N(0,I_d)$, we have when $d \gg 1$ that $\mathbb E|Y|\asymp \sqrt d$.
Without normalization, the regularization error would inherit this
ambient-scale factor and would be artificially large. 

Our contributions can be
summarized as follows:
\begin{itemize} 
\item 
(Theorems~\ref{thm:recovery_characterization} \& \ref{thm:tangential_recovery}) We show that under manifold-supported data, the semi-dual objective underconstrains the recovered transport maps. This creates off-manifold non-identifiability that can lead to spurious recovered maps, while the tangential on-manifold signal remains identifiable.

\item 
(Theorem~\ref{thm:map_level_convg} \& Proposition~\ref{thm:noise_tangential}) 
Building on the smoothing approach of~\citet{choi2025overcoming}, we give a geometric explanation of how smoothing removes off-support non-identifiability that leads to spurious recovery. We further strengthen prior theory by proving map-level recovery guarantees as $\varepsilon\downarrow 0$: full map recovery under regularity, and tangential (on-manifold) recovery in general.

\item
(Theorem~\ref{lem:theta_derivatives1}, Lemma~\ref{lem:no_uniform_Lip}) We prove that as $\varepsilon \to 0$, the reduced semi-dual objective can become ill-conditioned (the Hessian with respect to the network parameters blows up), explaining the practical sensitivity of training at very small noise.

\item 
(Theorem~\ref{thm:semi_discrete}) We derive a sample-dependent terminal noise $\varepsilon_{\mathrm{stat}}(N)$ as in \eqref{eq:estat} that balances regularization bias, optimization stability, and finite-sample error, providing a concrete stopping rule for annealing. 

\end{itemize}

\section{Preliminaries}
\label{sec:pre}
The \textit{optimal transport map} \citep{monge1781memoire} is defined as a minimizer of a prescribed cost among all maps that push the source measure $\mu$ forward to the target measure $\nu$:
\begin{equation}\label{eq:ot_monge} 
    \mathcal{T}(\mu, \nu) := \inf_{T_\# \mu = \nu}  \left[ \int_{\mathcal{X} } c(x,T(x)) d \mu (x) \right].
\end{equation}
Here, $T_\#\mu$ denotes the pushforward of $\mu$ under $T$, and the constraint $T_\#\mu=\nu$ means that $T$ transports $\mu$ to $\nu$. It is well-known that the existence of an optimal transport map is not always guaranteed, and \citet{Kantorovich1948} proposed the following convex relaxation of \eqref{eq:ot_monge}:
\begin{equation} \label{eq:Kantorovich}
    C(\mu,\nu):=\inf_{\pi \in \Pi(\mu, \nu)} \left[ \int_{\mathcal{X}\times \mathcal{Y}} c(x,y) d\pi(x,y) \right],
\end{equation}
We refer to the joint probability distribution $\pi \in \Pi(\mu, \nu)$ as the \textit{transport plan} between $\mu$ and $\nu$. Here $\Pi(\mu,\nu)$ is the set of probability measures with left marginal $\mu$ and right marginal $\nu$ (see \cite{villani,santambrogio}).

\paragraph{Semi-dual Neural OT}

The Kantorovich problem \eqref{eq:Kantorovich} admits a semi-dual formulation \citep[Theorem 5.10]{villani}, \citep[Proposition 1.11]{santambrogio}:
\begin{equation} \label{eq:kantorovich-semi-dual}
     S(\mu,\nu):= \!\!\sup_{V} \left[ \int_\mathcal{X} V^c(x)d\mu(x) \!+\!\! \int_\mathcal{Y} V(y) d\nu (y) \right],
\end{equation}
where $V^{c}$ denotes the $c$-transform of $V$, i.e., 
\begin{equation} \label{eq:def_c_transform}
  V^c(x)=\underset{y\in \mathcal{Y}}{\inf}\left[ c(x,y) - V(y) \right].
  \vspace{-5pt}
\end{equation}
Replacing $V^{c}$ in \eqref{eq:kantorovich-semi-dual} by
\begin{equation}
    V^c(x)=\inf_{T: \mathcal{X} \rightarrow \mathcal{Y}}\left[ c(x,T(x)) - V(T(x)) \right]
\end{equation}
yields the max--min formulation \eqref{eq:snot_intro}. The resulting objective has been used to learn an optimal transport map with neural parameterizations \citep{otm, fanscalable, otmICNN}. See related works in   Appendix~\ref{app:related_works}.

\paragraph{Spurious solutions in SNOT}

Although \eqref{eq:snot_intro} is naturally derived from \eqref{eq:kantorovich-semi-dual}, existing SNOT approaches may admit max--min solutions that include not only the desired OT map but also spurious solutions \citep{otm}. In \citet[Theorem 3.1]{choi2025overcoming}, consistency of SNOT was established under the assumption that $\mu$ does not assign positive mass to sets of Hausdorff dimension $\le d-1$. When $\mu$ is supported on a lower-dimensional set, however, the learned map $T$ may fail to satisfy the desired relation $T_{\#}\mu=\nu$; see the examples in \citep[Section 3.2]{choi2025overcoming}.

\section{Geometric Origins of Spurious Solutions in Neural OT}\label{sec:semidual_failure}
\paragraph{Spurious solutions in the minimax formulation.}
We study the appearance of \emph{spurious solutions} in the minimax formulation~\eqref{eq:snot_intro} by fully characterizing its optimal solutions.
We say that a pair $(V^*,T_{\mathrm{rec}})$ is an optimal solution of~(1) if
\begin{align}
\label{eq:opt_pair_V}
V^* &\in \argmax_{V \in C(\cY)}\; \inf_{T:\cX\to \cY}\, \mathcal{L}(V,T), \\[2pt]
\label{eq:opt_pair_T}
T_{\operatorname{rec}} &\in \argmin_{T:\cX\to \cY}\, \mathcal{L}(V^*,T).
\end{align}

Here we always assume that $\cX , \cY \subset \R^d$ are compact sets.

For later use, given a fixed potential $V$, we say that $T_V$ is an optimal recovery map with respect to $V$ if
\begin{equation}\label{eq:opt_TV_def}
T_V \in \argmin_{T:\cX\to \cY}\, \mathcal{L}(V,T).
\end{equation}

We begin by characterizing what it means to optimally recover a map $T_V$ for a fixed potential $V$.

\begin{theorem}[Full characterization of the recovery step]\label{thm:recovery_characterization}
For any $V\in C(\cY)$ and cost $c\in C(\cX\times \cY)$, then there exists a $T_V$ that satisfies \eqref{eq:opt_TV_def}. Furthermore, one has that $T_V$ satisfies \eqref{eq:opt_TV_def}
if and only if for $\mu$-a.e. $x$
\begin{equation}\label{eq:recovery_constraint}
c\bigl(x,T_V(x)\bigr) - V\bigl(T_V(x)\bigr) = V^c(x).
\end{equation} 
\end{theorem}

Theorem~\ref{thm:recovery_characterization} explains why spurious solutions arise. Indeed, the recovery condition~\eqref{eq:recovery_constraint} is enforced only $\mu$-almost everywhere. We will see below that under the manifold assumption on $\mu$, this condition identifies only the recovered map’s tangential component along the data manifold, leaving the normal component unconstrained.

\begin{theorem}[Tangential recovery]\label{thm:tangential_recovery}
Assume that there exists a smooth compact manifold $\cM \subset \R^d$ with $\text{dim}(\cM)=m$ such that $\mu \ll \mathcal{H}^m  \llcorner \cM $. Let $(V^*,T_{\operatorname{rec}})$ be an optimal solution of~(1). Then if the cost $c(x,y) \in C^1(\R^d \times \R^d )$ and an optimal transport map $T_{\mathrm{opt}}$ exists, then for $\mu$-a.e.\ $x$,
\begin{equation} \label{eq:tangential_recovery}
   \Pi_{T_x\mathcal{M}}\, \nabla_x c\bigl(x,T_{\operatorname{rec}}(x)\bigr)
=
\Pi_{T_x\mathcal{M}}\, \nabla_x c\bigl(x,T_{\mathrm{opt}}(x)\bigr).
\end{equation}
\end{theorem}

Here, $\mathcal{H}^m$ denotes the $m$-dimensional Hausdorff measure and
$\mathcal{H}^m\llcorner \mathcal{M}$ denotes its restriction to $\mathcal{M}$.
In addition, $T_x\mathcal{M}$ denotes the tangent space of $\mathcal{M}$ at $x$, and
$\Pi_{T_x\mathcal{M}}$ denotes the orthogonal projection onto $T_x\mathcal{M}$.
 
In the quadratic cost case, Theorem~\ref{thm:tangential_recovery} implies for $\mu$-a.e. $x$
\begin{equation}
\Pi_{T_x\mathcal{M}}\, T_{\operatorname{rec}}(x)
=
\Pi_{T_x\mathcal{M}}\, T_{\mathrm{opt}}(x)
\end{equation}
meaning that the recovered map agrees with the true optimal transport
\emph{along directions tangent to the data manifold}, while discrepancies may occur
in directions normal to $\mathcal M$.\smallskip

\noindent\textbf{Practical takeaway.}
Theorems~\ref{thm:tangential_recovery} show that when $\mu$ is supported on a lower-dimensional set, the semi-dual objective is underconstrained.
This underconstraint manifests in directions normal to the data manifold, rendering the recovery
step non-identifiable: there may exist multiple recovery maps that attain the same objective value
while disagreeing substantially as transport maps.
From a learning perspective, this implies that gradient-based optimization can converge to
distinct solutions that fit the objective equally well but behave arbitrarily off the data manifold, leading to unstable or spurious learned maps. Empirically, this gap is visible in Table~\ref{tab:tang_vs_normal}: the learned neural OT map can match the map \emph{tangentially} while incurring much larger \emph{normal} error.

We now illustrate that the normal component of recovered maps is non-identifiable with a simple example:
\begin{example}[Non-identifiability in semi-dual recovery]\label{ex:failure}
Let $\mu$ and $\nu$ be the uniform measures on
\[
\supp(\mu)=[-1,1]\times\{0\},\qquad \supp(\nu)=\{0\}\times[-1,1].
\]
For the quadratic cost $c(x,y)=\tfrac12|x-y|^2$, any measurable map
$T=(T_1,T_2)$ with $T_\#\mu=\nu$ is optimal. This constraint implies for $\mu$-a.e. $x$
\[ T_1(x)=0. \]
Moreover, \citet[Example~1]{choi2025overcoming} showed that any measurable
$T_{\mathrm{rec}}:\supp(\mu)\to\supp(\nu)$ is an optimal recovery map.
Hence the tangential component $(T_{\mathrm{rec}})_1=(T_{\mathrm{opt}})_1$ $\mu$-a.e.\ for every optimal map $T_{\mathrm{opt}}$
(tangential to $\supp(\mu)$, i.e.\ along the $x_1$ direction), while the normal component $(\cdot)_2$ may differ arbitrarily.

\end{example}

\section{Regularization Removes Spurious Solutions and Stabilizes Learning}
\label{section:regularization}
\subsection{Off-manifold regularization}\label{subsec:offmanifold_reg}

The failure mode above stems from the fact that the recovery constraint~\eqref{eq:recovery_constraint}
is enforced only $\mu$-almost everywhere, i.e., only on the data support.
A natural way to remove this ambiguity is to replace $\mu$ by a full-dimensional approximation,
so that the constraint is imposed throughout ambient space.

Concretely, we smooth $\mu$ by additive noise:
\begin{equation} \label{eq:smoothing}
    \mu_\varepsilon := \operatorname{Law}(X+\varepsilon Y),
\end{equation}
where $X\sim\mu$ and $Y$ is independent of $X$ with density $\rho\ge 0$ and $\int \rho = 1$.
Equivalently, $\mu_\varepsilon=\mu*\rho_\varepsilon$, where
$\rho_\varepsilon(z)=\varepsilon^{-d}\rho(z/\varepsilon)$.
For our map-level recovery analysis, we assume that $\mu_{\varepsilon}$ is compactly supported (equivalently, that $\rho$ is compactly supported) to enable compactness arguments
(Theorem~\ref{thm:map_level_convg} and Proposition~\ref{thm:noise_tangential}).
In contrast, the finite-sample convergence rates in Theorem~\ref{thm:semi_discrete} only require $\mathbb{E}[|Y|]<\infty$,
and thus apply immediately to Gaussian smoothing $Y\sim\mathcal{N}(0,I_d)$ as commonly used in practice.

From an optimization viewpoint, smoothing acts as an implicit regularizer.
In the unsmoothed problem, the objective constrains the recovered map only on $\supp(\mu)$, so the loss
can be insensitive to how the map (or potential) is extended in directions normal to the data manifold.
These off-manifold degrees of freedom create non-identifiability and can lead to families of
equally optimal solutions during training.
Replacing $\mu$ by $\mu_\varepsilon\ll\mathcal L^d$ expands the set of points at which the constraint \eqref{eq:recovery_constraint} is
enforced, which reduces this degeneracy and stabilizes both the learned potentials and the
recovered maps. Consistent with this interpretation, Table~\ref{tab:tang_vs_normal} shows that smoothing reduces the \emph{normal} error.

\subsection{Map consistency under regularization}\label{subsec:map_consistency}

The discussion above shows that smoothing resolves the non-identifiability of the recovery step by enforcing the recovery constraint \eqref{eq:recovery_constraint} off the data manifold. In the full-dimensional setting, this mechanism uniquely identifies the transport map. Indeed, Lemma~\ref{lem:full_recovery} shows that in the favorable full-dimensional regime, that Theorem~\ref{thm:tangential_recovery} implies
the semi-dual objective is identifiable: optimizing the objective uniquely determines the transport map.

We now study the vanishing-regularization limit $\varepsilon\downarrow 0$. 
While~\citet{choi2025overcoming} proved smoothing restores correct \emph{plan-level} 
convergence (controlling optimal plans), this does not 
establish convergence of the maps $T_\varepsilon$ themselves, nor optimality of 
any limiting map. We address this gap 
by proving \emph{map-level} guarantees: under regularity assumptions (a best-case 
scenario), the recovered maps $T_\varepsilon$ converge to the correct optimal limiting 
transport map as $\varepsilon\downarrow 0$.

\begin{theorem}[Map-level stability under regularization]\label{thm:map_level_convg}
Let $c(x,y)=\tfrac 12 |x-y|^2$. Assume that $K\subset\R^d$ be compact, and let
$\mu_\varepsilon\in\mathcal P_2(K)$ be such that $\mu_{\eps} \ll \mathcal{L}^d$  and $\mu_\varepsilon\rightharpoonup\mu$.
Fix $\nu\in\mathcal P_2(K)$.

For each $\varepsilon>0$, let $(V_\varepsilon,T_\varepsilon)$ be an optimal solution of
\eqref{eq:snot_intro} with source measure $\mu_\varepsilon$ and target measure $\nu$. Then $(V_{\eps}^{cc},T_{\eps})$ is also an optimal solution of
\eqref{eq:snot_intro} with the source measure $\mu_\varepsilon$ and target measure $\nu$.

Then, if one normalizes $V_{\eps}^{cc}$ by setting $V_{\eps}^{cc}(0)=0$, then along a subsequence, $V^{cc}_\varepsilon$ converges uniformly on $K$.

Moreover, if the limiting potential is $V_0$, then:
\begin{enumerate}[label=(\arabic*)]
\item If $V_0^c$ is differentiable $\mu$-a.e., then
\[
T_0(x):=x-\bigl(\nabla V_0^c(x)\bigr)
\]
is an optimal Monge transport map from $\mu$ to $\nu$.

\item If, in addition, each $V_\varepsilon^c$ is differentiable on $K$, then for $\mu$-a.e.\ $x\in K$
there exists a sequence $x_\varepsilon\to x$ such that
\[
T_\varepsilon(x_\varepsilon)\to T_0(x).
\]
Moreover, the pair $(V_0,T_0)$ is an optimal solution of~\eqref{eq:snot_intro} with source measure $\mu$
and target measure $\nu$.
\end{enumerate}
\end{theorem}

\begin{remark} \label{rmk_c_transform}
Assume that $K$ is connected, $\mathcal{L}^d(\partial K)=0$, $V_\varepsilon$ is Lipschitz continuous, $\nu\ll\mathcal{L}^d$, and $\nu \in \mathcal{P}(K)$.
Then we have that $\nabla V_{\eps} = \nabla V_{\eps}^{cc}$ for Lebesgue a.e. $x$. In particular, $V_{\eps}$ and $ V_{\eps}^{cc}$ differ by a constant. See Section~\ref{proof_remark}.
\end{remark}

Without the differentiability assumption, (i) even if a limiting map exists, it may fail to even be a transport map i.e.
$(T_0)_{\#} \mu \neq \nu$, and (ii) the unsmoothed problem may admit no optimal transport Monge maps (see Example~\ref{maps_converge_wrongly}).

Theorem~\ref{thm:map_level_convg} characterizes a best-case regime in which smoothing yields full recovery of the optimal transport map as the noise level vanishes. We now prove an unconditional guarantee: even without such regularity, smoothing still preserves the correct on-manifold transport signal, in the sense that the tangential component along the data manifold can always be recovered.

\begin{proposition}[Tangential recovery on manifolds as noise vanishes]\label{thm:noise_tangential}
Using the setting and notation of Lemma~\ref{lem:graphical_convg}, assume in addition that there exists a smooth
compact manifold $\cM\subset\R^d$ of dimension $m$ such that $\mu \ll \mathcal H^m \llcorner \mathcal M$. Let $V_0$ denote the limiting potential. Then $V_0$ is optimal in the semi-dual sense, i.e.\ it maximizes~\eqref{eq:opt_pair_V}.

Moreover, suppose an optimal transport map $T_{\operatorname{opt}}$ exists between $\mu$ and $\nu$.
Then for $\mu$-a.e.\ $x\in\cM$,
\[
\nabla_{\cM} V_0^c(x)
=
\Pi_{T_x\cM}\,\nabla_x c\bigl(x,T_{\operatorname{opt}}(x)\bigr).
\]
\end{proposition}

In the quadratic cost case, this implies
\[ \nabla_{\cM} \left( \frac{1}{2} |x|^2-V^c_0 \right) = \Pi_{T_x \cM} (T_{\operatorname{opt}}(x)) \] By combining this with Theorem~\ref{thm:tangential_recovery} we see that for the quadratic cost case that  $\nabla_{\cM} \left( \frac{1}{2} |x|^2-V^c_0 \right)$ is the tangential component of any recovered map.

\noindent \textbf{Practical takeaway.}
Adding noise during training cannot erase the identifiable on-manifold transport signal: as $\varepsilon\downarrow 0$,
the learned potentials recover the correct tangential component of the optimal transport map.
Moreover, in favorable regularity regimes, the same smoothing scheme yields full recovery of the optimal transport map.

\section{Bias-Stability Tradeoff for Noise Scheduling}\label{sec:noise_schedule}
\subsection{Plan stability and quantitative rates}\label{subsec:plan_stability}

We next establish quantitative \emph{plan stability} bounds, controlling how optimal OT couplings change under perturbations of the input distribution.
Unlike the weak convergence guarantees of \citet{choi2025overcoming} under smoothing, we provide explicit \emph{finite-sample} rates for learning couplings when the target $\nu$ is sufficiently regular.
In particular, our bounds make explicit the \emph{bias-estimation tradeoff} induced by smoothing: the noise level $\varepsilon$ introduces approximation bias, while the sample size $N$ governs statistical error. 
Throughout, we specialize to the quadratic cost $c(x,y)=\tfrac12|x-y|^2$. \smallskip

For $\rho_1,\rho_2\in\mathcal P_2(\R^d)$, let $\pi(\rho_1,\rho_2)$ denote an optimal coupling between $\rho_1$ and $\rho_2$.
We assume regularity of the target $\nu$, but allow the source $\mu$ to be arbitrary (possibly supported on a lower-dimensional set), so Monge maps need not exist and we work directly with plans. \smallskip

In practice, $\mu$ is observed only through samples and we learn from noisy samples. Our goal is to bound the learned plan error by decomposing it into (i) a smoothing term that vanishes as $\varepsilon\downarrow 0$ (bias) and (ii) a sampling term that vanishes as $N\to\infty$ (estimation).
Concretely, we consider
\[
\mu^{\eps}_N := \frac{1}{N} \sum_{i=1}^{N} \delta_{X_i + \eps Y_i},
\]
where $X_1,\dots,X_N \sim \mu$ and $Y_1,\dots,Y_N \sim Y$ are i.i.d.\ and independent of each other.

\begin{theorem}[Finite-sample stability with smoothed empirical sources]\label{thm:semi_discrete}
Assume the target measure $\nu$ satisfies the regularity conditions of Theorem~\ref{thm:quant_general}. Let $\mu$ be compactly supported and suppose $\supp(\mu)\subset \mathcal M$ for a compact $m$-dimensional smooth manifold
$\mathcal M\subset\R^d$. Assume that $Y$ satisfies $\mathbb{E}[|Y|]< \infty$, then for $\mu_\varepsilon:=\operatorname{Law}(X+\varepsilon Y)$
\[
\mathbb E\,W_2\!\big(\pi(\mu,\nu),\pi(\mu_N^\varepsilon,\nu)\big) \lesssim (\mathbb{E} W_1(\mu,\mu^{\eps}_N))^{\alpha(p)} 
\] \[ \lesssim
(\mathbb{E}[|Y|] \cdot \varepsilon)^{\alpha(p)}  
\;+\;
\begin{cases}
N^{-\alpha(p)/2}, & m=1,\\[2pt]
(\log N/N)^{\alpha(p)/2}, & m=2,\\[2pt]
N^{-\alpha(p)/m}, & m\ge 3.
\end{cases}
\]
Where $p\ge \max(4,d)$ and  $\alpha(p):=\frac{p}{6p+16d}$.
\end{theorem}

\begin{remark}\label{rem:alpha_limit}
As $p\to\infty$, $\alpha(p)\to 1/6$. Larger stability exponents (up to $1/2$) are known under much stronger assumptions, e.g.\ Lipschitz regularity of the optimal transport map; see \citet{gigli2011holder,merigot2020quantitative}.
Such assumptions are typically violated in learning, especially when the source is supported on a lower-dimensional manifold.
\end{remark}

In Theorem~\ref{thm:lower_bounds} we obtain unconditional lower bounds on learning the optimal plan.

\paragraph{Takeaway.}
Theorem~\ref{thm:semi_discrete} and Theorem~\ref{thm:lower_bounds} show that learning OT plans from smoothed samples is fundamentally bottlenecked by estimating the smoothed source in Wasserstein distance:
the upper bound transfers $\mathbb E\,W_1(\mu_\varepsilon,\mu_N^\varepsilon)$ to
$\mathbb E\,W_2\!\big(\pi(\mu,\nu),\pi(\mu_N^\varepsilon,\nu)\big)$, while the lower bound shows this dependence is unavoidable. \smallskip

For fixed $N$, decreasing $\varepsilon$ helps only until a statistical floor is reached, after which further annealing does not change the asymptotic error and can worsen conditioning (see Section~\ref{sec:noise_schedule}).
Moreover, under the manifold assumption, $\mathbb E\,W_1(\mu_\varepsilon,\mu_N^\varepsilon)$ scales with the intrinsic dimension $m$ (not the ambient $d$), so plan learning can avoid the curse of dimensionality under the common manifold assumption.

\begin{figure}[t]
    \centering
    \includegraphics[width=0.4\textwidth]{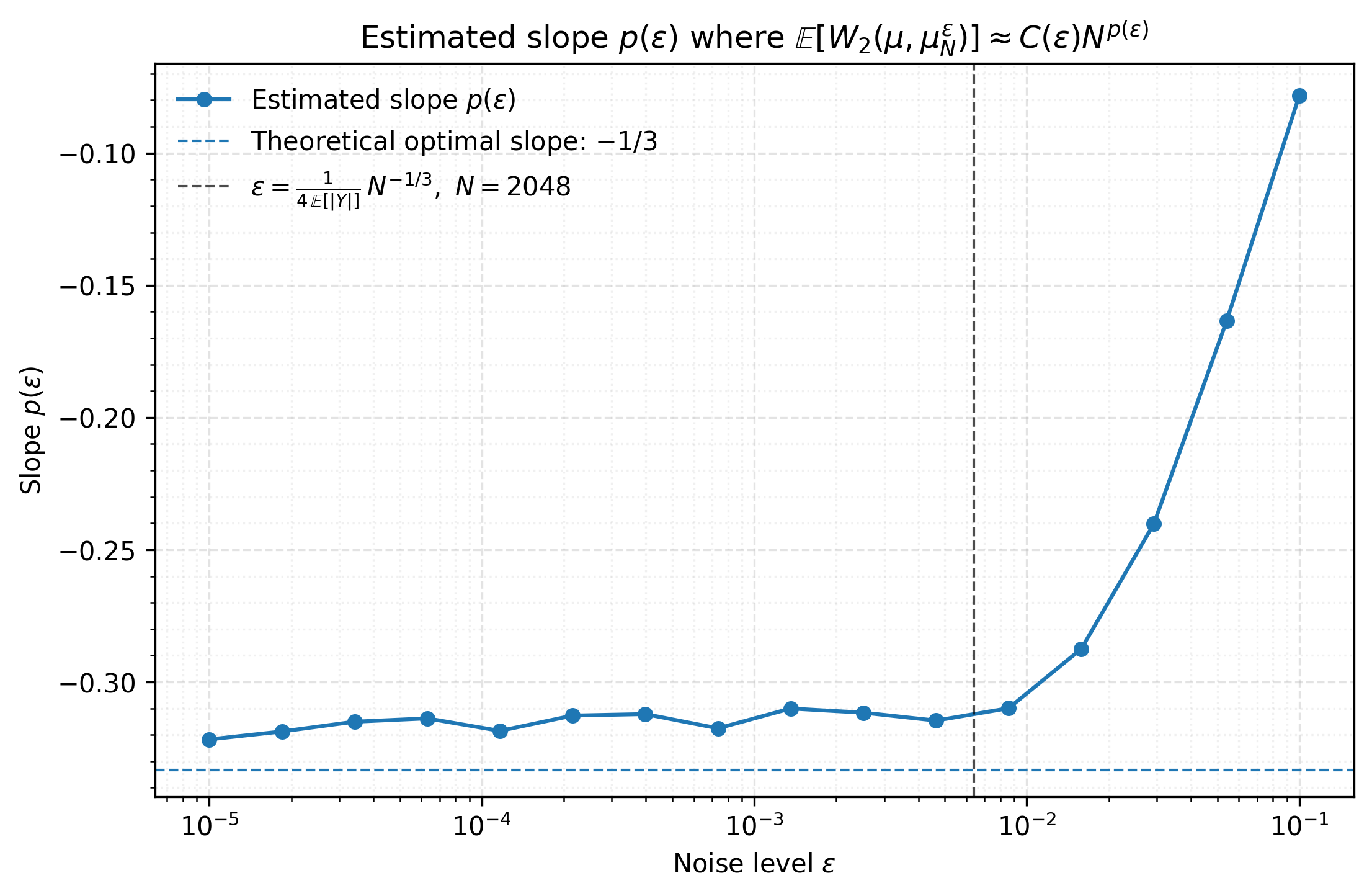}
    \caption{\textbf{Estimated slope in $\mathbb{E}[W_2(\mu, \mu_N^\varepsilon)] \approx C(\varepsilon) N^{\rho(\varepsilon)}$.} We set $\mu = \mathrm{Unif}([-1,1]^3) \otimes (\delta_0)^7 \subset \mathbb{R}^{10}$ ($d=10, m=3$). For $\varepsilon \leq \varepsilon_{\mathrm{stat}} \approx 10^{-2}$, the slope matches the theoretical optimal rate $\rho = -1/3$, confirming that OT plan learning avoids the curse of dimensionality under manifold structure. For $\varepsilon > \varepsilon_{\mathrm{stat}}$, the rate degrades due to finite-sample estimation error. Minor deviations from $-1/3$ are attributed to discretizing $\mu$ in the discrete OT solver.}
    \label{fig:eps_vs_slopes}
    \vspace{-8pt}
\end{figure}

The results in Figure~\ref{fig:eps_vs_slopes} confirm our theoretical findings. Theorem ~\ref{thm:semi_discrete} (and more explicitly \eqref{desired_bound_123}) implies that when the noise level $\eps \leq \eps_{\textnormal{stat}}$, the source measure error decays at a rate of $N^{-1/m}$ (independent of ambient dimension $d$). When $\eps \leq \eps_{\textnormal{stat}}$, Figure~\ref{fig:eps_vs_slopes} demonstrates this optimal rate. While when the noise level $\eps > \eps_{\textnormal{stat}}$, the slope degrades because the smoothing bias term dominates.

\subsection{Stability: small noise can degrade conditioning}\label{subsec:stability}

Regularizing the source distribution removes the off-manifold ambiguity from Section~\ref{sec:semidual_failure}, but introduces a tradeoff:
larger noise improves identifiability and stabilizes optimization, while smaller noise reduces regularization bias. As observed in \citet{choi2025overcoming}, stopping the annealing at a small constant $\varepsilon_{\min}>0$ can be more stable than continuing to $\varepsilon=0$. We justify this via (i) an \emph{optimization} effect—overly small noise induces sharp maps and an ill-conditioned loss landscape, and
(ii) a \emph{statistical} effect, there is a noise floor $\varepsilon_{\mathrm{stat}}(N)$ below which finite-sample estimation error dominates. Accordingly, we set the terminal noise level to $\varepsilon_K=\varepsilon_{\mathrm{stat}}(N)$ from \eqref{eq:estat}.

We now explain why training should stop at a nonzero terminal noise level $\varepsilon_K$.
Very small noise can induce \emph{sharp} (high-Jacobian) transport maps, which in turn create large curvature in the objective and make gradient-based training increasingly sensitive. \smallskip

As a first step, Lemma~\ref{lem:no_uniform_Lip} shows that if the unsmoothed problem admits no optimal Monge map, then the smoothed Monge maps $(T_\varepsilon)_{\varepsilon>0}$ satisfies
\[
\limsup_{\eps \rightarrow 0 } \|\nabla T_\varepsilon\|_{L^\infty(B_{R_0})}=+\infty
\] for some $R_0>0$. This is typical for manifold-supported data: smoothing restores existence but can induce sharp off-manifold extensions.

To illustrate Lemma~\ref{lem:no_uniform_Lip}, Example~\ref{ex:1d_stability_proxy} gives a 1D toy case where smoothing regularizes recovery: as $\varepsilon\downarrow 0$, $T_\varepsilon$ becomes steep, while larger $\varepsilon$ smooths the map and reduces $\|\nabla T_\varepsilon\|$. Thus, larger $\varepsilon$ can improve conditioning at the cost of bias.

To connect with training, we consider the reduced objective.
\[
\mathcal J(\theta):=\inf_T \mathcal L(V_\theta,T).
\]
Even if the infimum is attained exactly, the curvature of $\mathcal J$ in $\theta$ depends on derivatives of the induced map $T_\theta$.

\begin{theorem}[Curvature amplification in the reduced objective]\label{lem:theta_derivatives1}
Under mild regularity and uniqueness assumptions, the Hessian $\nabla^2_{\theta\theta}\mathcal J(\theta)$ contains the amplification factor
\[ \int_{\cX} G_{\theta}(x)^{\top} \nabla_x T_{\theta}(x) G_{\theta}(x) d\mu(x), \] where $G_{\theta}(x) = \nabla_x \nabla_{\theta} V_{\theta}(T_{\theta}(x))$.
\end{theorem}

An explicit formula for $\nabla^2_{\theta\theta}\mathcal J(\theta)$ is given in Lemma~\ref{lem:theta_derivatives}.

\paragraph{Practical takeaway.}
Lemma~\ref{lem:no_uniform_Lip} and Theorem~\ref{lem:theta_derivatives1} suggest a simple mechanism:
as $\varepsilon$ decreases, the transport map can become sharp (large $\|\nabla T_\varepsilon\|$), which can in turn make the loss landscape more ill-conditioned for gradient methods. This motivates using a nonzero terminal noise in practice, rather than pushing $\varepsilon\to 0$ in a finite-sample pipeline.

\section{Experiments} \label{sec:experiments}
We evaluate our comprehensive theoretical analysis of the semi-dual Neural OT through a series of numerical experiments. The goal of our experiments is to verify the following key claims of our analysis:
\begin{itemize}[topsep=-1pt, itemsep=-1pt]
    \item In Sec \ref{sec:exp_tangential_normal}, we demonstrate that smoothing identifies transport signals in normal directions normal to the data manifold where unregularized methods fail.
    \item In Sec \ref{sec:exp_NOT_accuracy}, we evaluate our principled noise schedule for the Neural OT Plan by evaluating the accuracy.

    \item In Sec~\ref{sec:terminal_noise_validation}, we demonstrate that annealing 
below $\varepsilon_{\text{stat}}(N)$ does not improve the statistical convergence rate of the learned transport maps.
    
    \item In Sec \ref{sec:exp_bias_stability}, we  demonstrate the bias–stability tradeoff: overly small $\varepsilon$ can substantially increases the iterations needed for convergence. 
\end{itemize}
For implementation details of experiments, please refer to Appendix \ref{app:implementation_details}.

\begin{figure*}[t]
    \centering
    \begin{subfigure}[b]{0.3\textwidth}
        \centering
        \includegraphics[width=\textwidth]{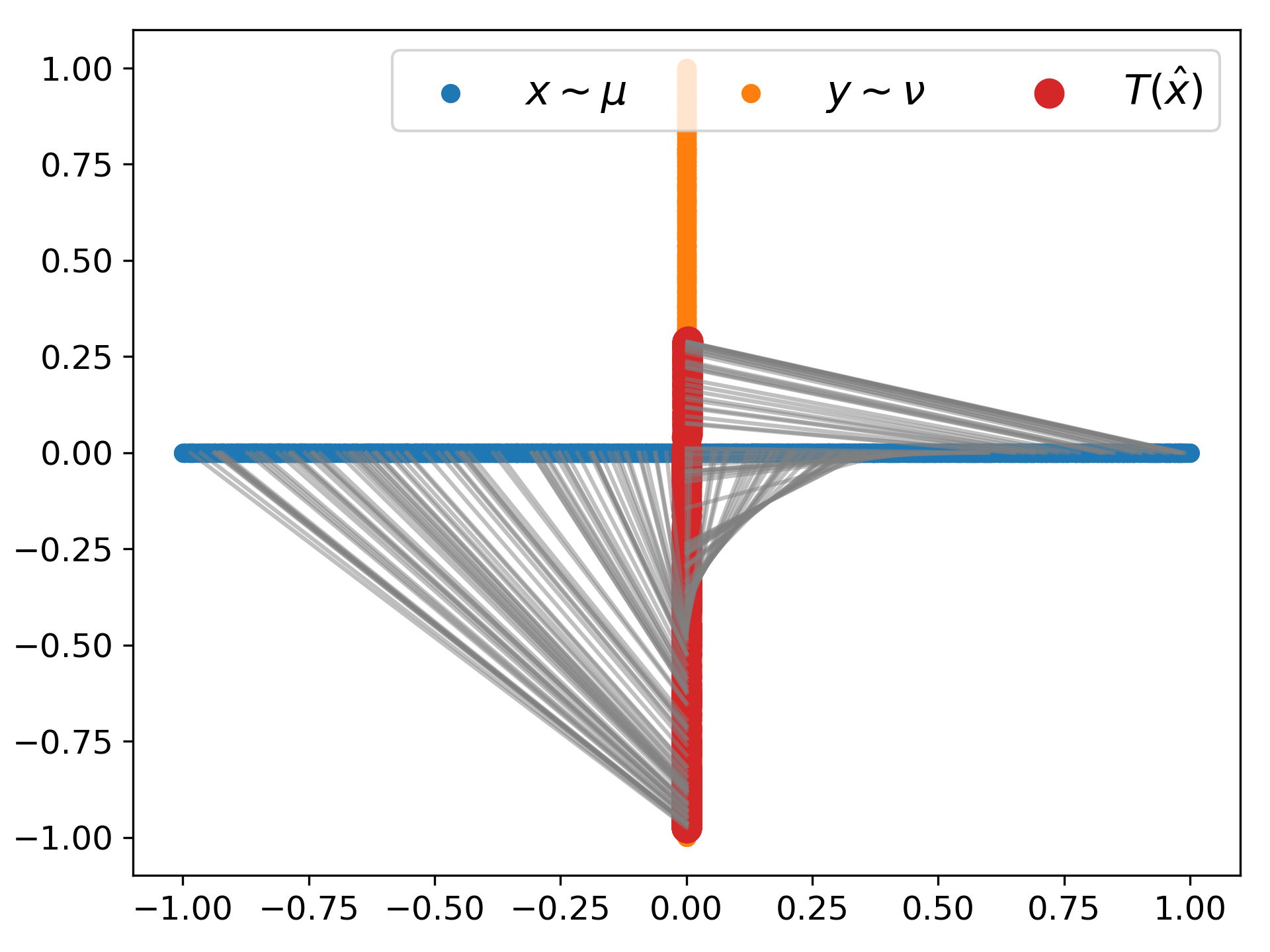}
        \caption{OTM}
        \label{fig:tang_vs_normal_OTM}
    \end{subfigure}
    \hfill
    \begin{subfigure}[b]{0.3\textwidth}
        \centering
        \includegraphics[width=\textwidth]{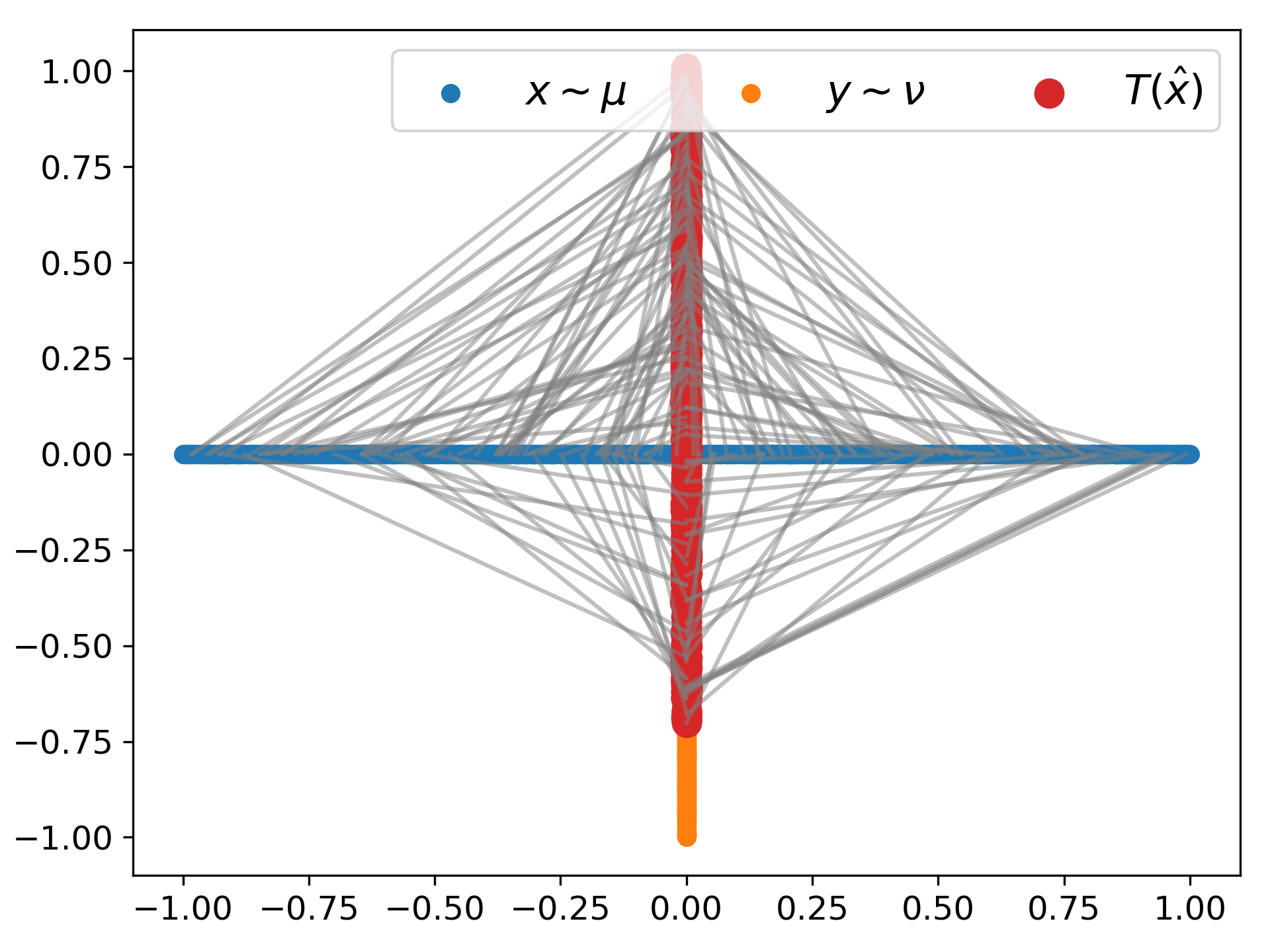}
        \caption{OTP}
        \label{fig:tang_vs_normal_OTP}
    \end{subfigure}
    \hfill
    \begin{subfigure}[b]{0.3\textwidth}
        \centering
        \includegraphics[width=\textwidth]{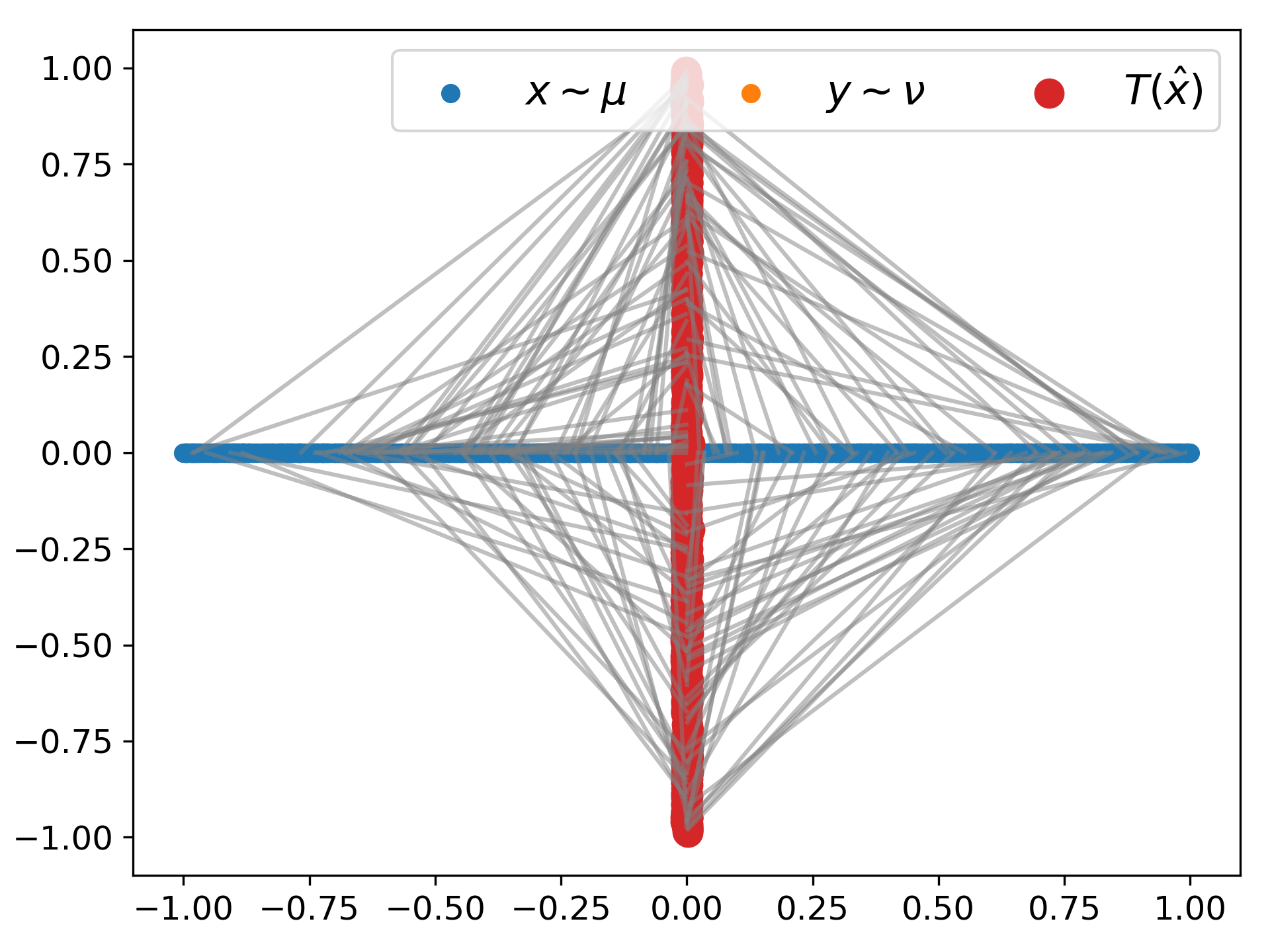}
        \caption{Ours}
        \label{fig:tang_vs_normal_ours}
    \end{subfigure}
    \caption{\textbf{Comparison of the transport map recovery on the Perpendicular Case.} All SNOT solvers are accurate in the tangential component ($x$-axis). However, only the smoothing-based methods (OTP and Ours) successfully recover the normal component ($y$-axis), with our principled approach achieving the highest accuracy.
    }
    \label{fig:tang_vs_normal}
    \vspace{-8pt}
\end{figure*}

\paragraph{Our Annealing Strategy}
We propose a principled, convergence-rate-optimal noise scheduling for the Optimal Transport Plan (OTP) model \citep{choi2025overcoming}. OTP regularizes singular source measures $\mu$ via Gaussian smoothing ($\mu_{\varepsilon}$) to eliminate spurious solutions and gradually decrease the noise level to approximate the original optimal transport plan $\pi^{\star}$. While the original framework utilizes heuristic decay rates inherited from diffusion model schedules (VP SDE) \citep{scoresde}, our strategy is grounded in the rate-optimal noise level $\varepsilon_{\mathrm{stat}}(n)$ i.e. \eqref{eq:estat}. This choice balances regularization bias against finite-sample estimation error. Here we use $Y = \mathcal{N}(0,I)$ with $n$ denoting the cumulative number of observed samples during training, calculated as $n = \text{training iterations} \times \text{batch size}$. To ensure optimization stability and prevent the input distribution from shifting continuously for the neural network, we employ a \textbf{stepwise decreasing rule}. Furthermore, we introduce a minimum noise level hyperparameter, $\varepsilon_{\mathrm{min}}$, to prevent the schedule from decreasing below a threshold that would degrade conditioning. The effective noise level is thus $\varepsilon_{\mathrm{eff}}(n) = \max(\varepsilon_{\mathrm{stat}}(n), \varepsilon_{\mathrm{min}})$.

\paragraph{Baselines}
We compare three SNOT approaches: (1) OTM (Vanilla) \citep{otm, fanscalable} (without any smoothing regularization), (2) OTP \citep{choi2025overcoming}, (3) Our approach (OTP with Rate-optimal annealing $\varepsilon_{\mathrm{eff}}(\cdot)$).

\begin{table}[t]
    \centering
    \caption{\textbf{Recovery accuracy in the Perpendicular Case.} Smoothing-based solvers (OTP, Ours) successfully recover the normal transport component where unregularized OTM fails. \textbf{S}: Smoothing; \textbf{O}: Rate-optimal scheduling.}
    \label{tab:tang_vs_normal}
    \scalebox{0.9}{
    \begin{tabular}{c c c c c}
        \toprule
        \multirow{2}{*}{Model} & \multirow{2}{*}{\textbf{S}} & \multirow{2}{*}{\textbf{O}} & \multicolumn{2}{c}{Error}\\
        \cmidrule{4-5}
        & & & Tangential Error ($\downarrow$) & Normal Error ($\downarrow$)\\
        \midrule
        OTM & \xmark & \dashmark & $3.405 \times 10^{-6} $  & 0.195 \\
        OTP & \cmark & \xmark & $1.192 \times 10^{-6} $ & 0.073 \\
        Ours & \cmark & \cmark & $\textbf{3.283} \times \mathbf{10^{-7}}$ & \textbf{0.005} \\
        \bottomrule
    \end{tabular}}
    \vspace{-8pt}
\end{table}

\subsection{Spurious solutions: tangential vs. normal error} \label{sec:exp_tangential_normal}
Theorem \ref{thm:tangential_recovery} proves that the semi-dual objective for SNOT \eqref{eq:snot_intro} is underconstrained when the source measure $\mu$ is supported on a lower-dimensional manifold. Specifically, only the tangential component of the recovered map $T_{\textnormal{rec}}$ is guaranteed to match to the ground-truth $T_{\textnormal{opt}}$ \eqref{eq:tangential_recovery}, allowing the normal component to behave arbitrarily. 
We test this using the Perpendicular synthetic dataset, where the source and target distributions are uniformly supported on $[-1, 1] \times \{0 \}$ and $\{0 \} \times [-1, 1]$\footnote{In this case, any map $T$ satisfying $T_{\#} \mu = \nu$ is an optimal transport map for the quadratic cost function.}. We decompose the map error into tangential and normal components. Let the neural OT model be $T_{\theta}(x) = (T_{\theta, 1}(x), T_{\theta, 2}(x))$. The \textbf{tangential error} is defined as $| \int \left( T_{\theta, 1}(x) - 0  \right) d \mu (x)|$, because the first component of the target support is always zero. The \textbf{normal error} is measured as $W_{2}^{2} ( (T_{\theta, 2})_{\#} \mu_{\eps_{\textnormal{min}}}, \, \textnormal{Unif} ([-1, 1]))$, which is the Wasserstein distance along the normal direction.

As shown in Figure \ref{fig:tang_vs_normal} and Table \ref{tab:tang_vs_normal}, the unregularized OTM fails to identify the normal component. While it matches the tangential component accurately ($x$-coordinate, with error $3.405 \times 10^{-6}$), its high normal error ($y$-coordinate, $0.195$) results in a spurious solution that leaves most of the target support uncovered.
In contrast, smoothing-based approaches (OTP and Ours) address this by replacing the singular source with a \textit{thickened} regular approximation $\mu_{\epsilon}$, imposing constraints in a neighborhood of the data. While the original OTP reduces the normal error to $0.073$, our rate-optimal approach further reduces it to $0.005$. These results empirically confirm that smoothing removes the off-support degrees of freedom that lead to spurious solutions.

\begin{table}[t]
    \centering
    \caption{\textbf{Low-dimensional manifold accuracy benchmark ($d=256$)}. This comparison demonstrates the stability of our principled noise scheduling when the ambient dimension is large but the intrinsic data manifold is low-dimensional.
    }
    \label{tab:accuracy_lowdim_mfd}
    \setlength{\tabcolsep}{3pt} 
    \scalebox{0.75}{
    \begin{tabular}{c l cc cc}
        \toprule
        \multirow{2}{*}{$m$} & \multirow{2}{*}{Model} & \multicolumn{2}{c}{Perpendicular} & \multicolumn{2}{c}{One-to-Many} \\
        \cmidrule(r){3-4} \cmidrule(l){5-6}
        & & $D_{cost}$ & $D_{target}$ & $D_{cost}$ & $D_{target}$ \\
        \midrule
        \multirow{3}{*}{2} & OTM & $0.483 \pm 0.255$ & $0.651 \pm 0.159$ & $\underline{0.397} \pm 0.163$ & $\underline{2.215} \pm 0.327$ \\
                           & OTP  & $\underline{0.116} \pm 0.062$ & $\textbf{0.079} \pm 0.003$ & $7.511 \pm 10.388$  & $8.278 \pm 11.123$ \\
                           & Ours & $\textbf{0.098} \pm 0.027$ & $\underline{0.131} \pm 0.012$ & $\mathbf{0.298} \pm 0.019$ & $\textbf{0.362} \pm 0.043$ \\
        \midrule
        \multirow{3}{*}{4} & OTM & $0.640 \pm 0.473$ & $1.089 \pm 0.635$ & $\underline{0.719} \pm 1.032$ & $\underline{1.190} \pm 1.731$ \\
                           & OTP  & $\underline{0.367} \pm 0.066$ & $\underline{0.216} \pm 0.032$ & $9.062 \pm 12.606$ & $9.877 \pm 13.407$ \\
                           & Ours & $\textbf{0.267} \pm 0.047$ & $\textbf{0.188} \pm 0.014$ & $\textbf{0.265} \pm 0.011$ & $\textbf{0.426} \pm 0.032$ \\
        \midrule
        \multirow{3}{*}{8} & OTM & $0.919 \pm 0.362$ & $4.339 \pm 0.524$ & $\underline{0.401} \pm 0.096$ & $\underline{1.219} \pm 0.778$ \\
                           & OTP  & $\underline{0.651} \pm 0.050$ & $\textbf{0.644} \pm 0.037$ & $3.062 \pm 2.946$ & $6.057 \pm 2.929$ \\
                           & Ours & $\mathbf{0.546 \pm 0.039}$ & $\underline{0.652} \pm 0.024$ & $\textbf{0.150} \pm 0.005$ & $\textbf{0.978} \pm 0.011$ \\
        \bottomrule
    \end{tabular}
    }
    \vspace{-8pt}
\end{table}

\subsection{Neural OT Plan accuracy evaluation} \label{sec:exp_NOT_accuracy}

We evaluate the numerical accuracy of the learned Neural OT models using a \textbf{low-dimensional manifold benchmark}. Here, the ambient dimension is fixed at $d=256$ and the manifold dimension is set to much smaller scales $m \in \{2, 4, 8 \}$. This setup is highly challenging to the Neural OT models because the data are concentrated on highly singular, low-dimensional manifolds. (See Appendix \ref{app:standard_manifold} for larger-dimensional manifold benchmark results where $m = d/2$.) We evaluate performance using two metrics: \textit{Transport cost error} $D_{cost} = | W^2_2 (\mu, \nu) - \int \Vert T_{\theta}(x) - x \Vert^2 d\mu(x) |$ and \textit{Target distribution error} $D_{target} = W^2_2 (T_{\theta \#} \mu, \nu)$. $D_{cost}$ assesses whether the model achieves the optimal transport cost, while $D_{target}$ measures how accurately the model generates the target distribution. 

The quantitative results are presented in Table \ref{tab:accuracy_lowdim_mfd}. 
In the Perpendicular dataset, both smoothing approaches (OTP and Ours) are successful compared to the non-smoothing counterpart (OTM), with our approach yielding superior or comparable results to OTP.
Furthermore, our method maintains significant stability in the One-to-Many dataset. Our rate-optimal noise scheduling results in a more moderate noise decay as the manifold dimension increases, providing a principled terminal noise level $\varepsilon_{\text{stat}}(N)$. However, the standard OTP model does not incorporate such adaptation. OTP results in the One-to-Many case demonstrate that when the noise level is set too low, the lack of sufficient smoothing regularization leads to unstable optimization, as described by our analysis of ill-conditioning at small noise levels.

\subsection{Validation of the terminal noise level}
\label{sec:terminal_noise_validation}

We validate that $\varepsilon_{\text{stat}}(N)$ identifies the statistical 
noise floor predicted by Theorem~\ref{thm:semi_discrete}. 
For $\mu = \text{Unif}([0,1]^3) \otimes \delta_0^2$, 
$\nu = \text{Unif}([0,1]^5)$, and $N = 320000$ samples ($m=3$), 
Figure~\ref{fig:neural_map_fit} shows error remains stable for 
$\varepsilon \lesssim \varepsilon_{\text{stat}}$ (the $N^{-1/3}$ 
convergence floor is reached) but increases sharply for 
$\varepsilon \gg \varepsilon_{\text{stat}}$ (bias dominates). 
This confirms $\varepsilon_{\text{stat}}(N)$ as a practical upper 
bound for noise annealing.
\subsection{Regularization effect of Smoothing} \label{sec:exp_bias_stability}

We validate the predictions of Theorem~\ref{lem:theta_derivatives1} 
and Lemma~\ref{lem:no_uniform_Lip}. 
For $\mu=\delta_0$ and $\nu=\mathcal{N}(0,I_{10})$. Thus $\mu_{\eps} = \mathcal{N}(0,\eps^2 I_{10})$ and the Monge map is given by $T_\varepsilon(x)=x/\varepsilon$ with $\nabla T_\varepsilon=\varepsilon^{-1}I_{10}$. 
Figure~\ref{fig:map_gradient_blow_up} confirms the learned map exhibits the same behavior: $\|\nabla T_{\theta,\varepsilon}\|_{L^\infty}$ grows 
rapidly as $\varepsilon$ decreases. Since Theorem~\ref{lem:theta_derivatives1} 
links $\nabla^2_{\theta\theta}\mathcal{J}(\theta)$ to $\nabla T_{\theta,\eps}$, this suggests poor conditioning. 
Figure~\ref{fig:mse_vs_n} is consistent with this poor conditioning: it shows smaller 
$\varepsilon$ requires more iterations and higher variance, validating that annealing below $\varepsilon_{\text{stat}}(N)$ degrades optimization.

\begin{figure}[t]
    \centering
    \includegraphics[width=0.35\textwidth]{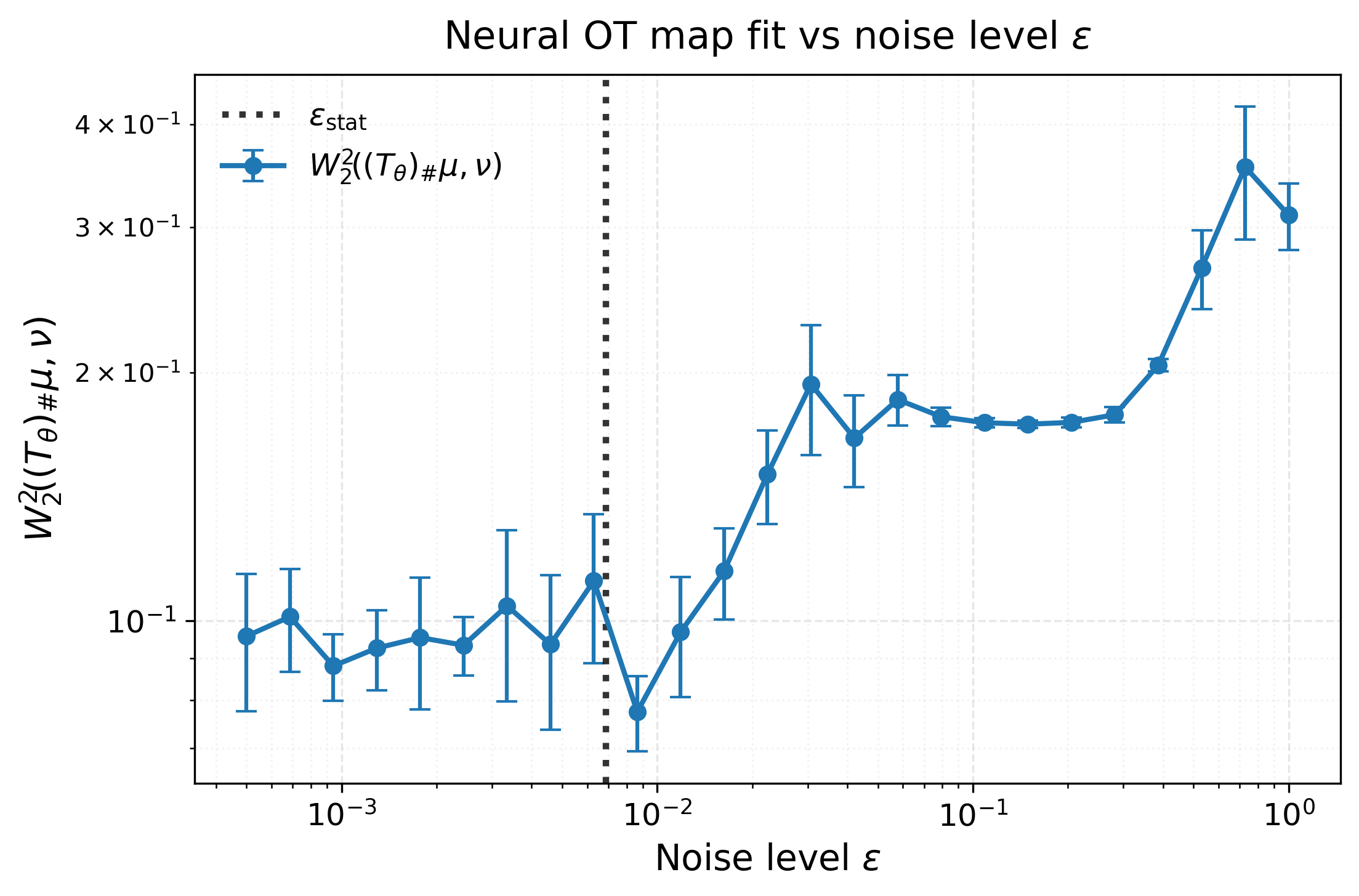}
    \caption{\textbf{Neural map error vs. $\varepsilon$.} Performance degrades for $\varepsilon \gg \varepsilon_{\text{stat}}$ 
(smoothing bias). Error remains stable for $\varepsilon \ll \varepsilon_{\text{stat}}$, 
confirming that reducing noise below this level does not improve the  statistical rate. Error bars denote sample standard deviations obtained through 5 trials.} 
    \label{fig:neural_map_fit}
    \vspace{-8pt}
\end{figure}

\begin{figure}[t]
    \centering
    \includegraphics[width=\columnwidth]{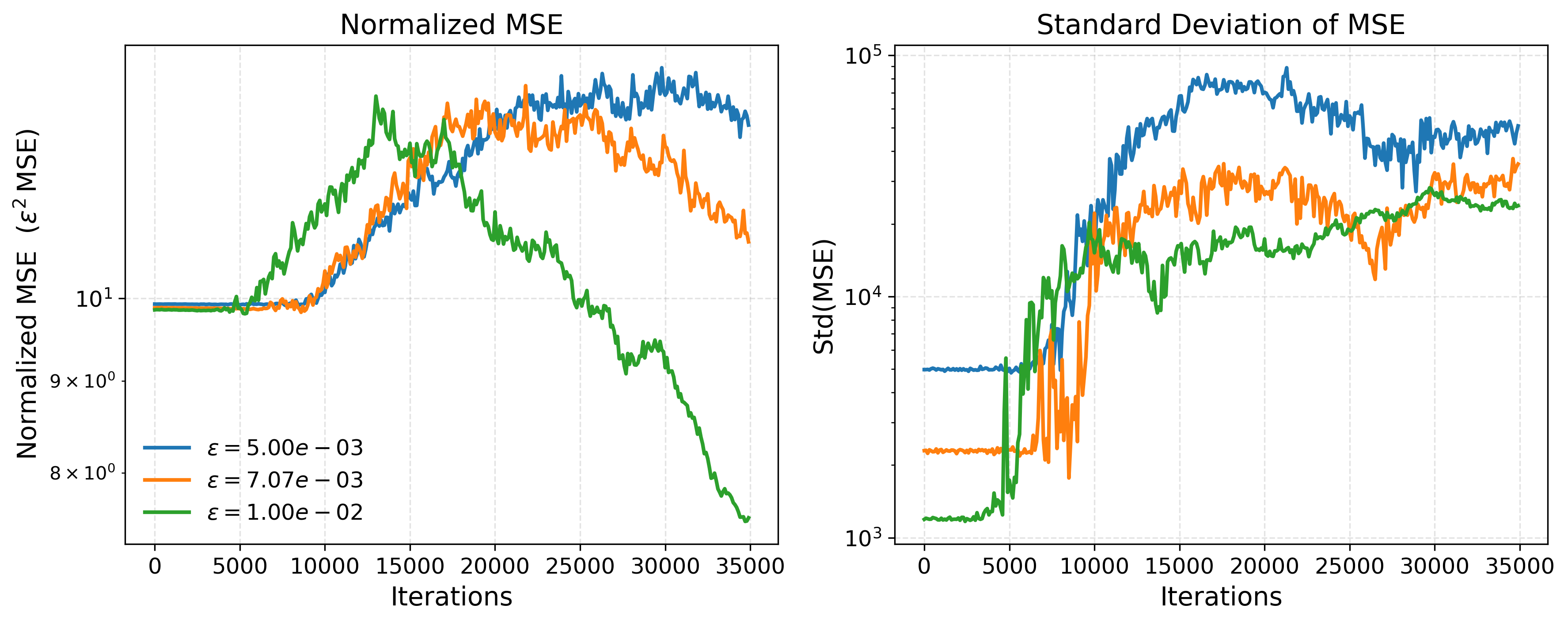}
    \caption{\textbf{Effect of $\varepsilon$ on convergence and variance.} Decreasing $\varepsilon$ induces slower convergence (\textbf{Left}) and higher variance across 10 runs (\textbf{Right}). These results empirically validate that small noise levels degrade optimization conditioning, justifying our terminal noise floor $\varepsilon_{\text{stat}}(N)$. Both panels use log-scaled $y$-axes.}
    \label{fig:mse_vs_n}
    \vspace{-8pt}
\end{figure}

\section{Conclusion}
This paper provides a geometric and statistical characterization of spurious solutions in semi-dual Neural Optimal Transport (SNOT). We identify that when data concentrate on a manifold, the objective is underconstrained in normal directions, though the tangential signal remains identifiable. To resolve this, we utilize additive-noise smoothing to remove off-manifold ambiguity and prove that recovered maps $T_{\epsilon}$ converge to the unique Monge map as smoothing vanishes. Our quantitative analysis establishes that plan estimation error scales with the intrinsic dimension $m$ rather than the ambient dimension $d$, bypassing the curse of dimensionality. We derive a principled terminal noise level $\eps_{\text{stat}}(N)$ that balances regularization bias against optimization stability. We validate our theory through a series of numerical experiments.

\section*{Acknowledgements}
RC is partially supported by NSF grant
DMS 2342349. JW was partially supported by the National Research Foundation of Korea (NRF) grant funded by the Korea government (MSIT) [RS-2024-00349646]. DK is partially supported by the National Research Foundation of Korea (NRF) grant funded by the Korea government (MSIT) (No. RS-2023-00252516 and No. RS-2024-00408003), the POSCO Science Fellowship of POSCO TJ Park Foundation, and the Korea Institute for Advanced Study. DK thanks Wilfrid Gangbo for helpful discussions on this project and the Center for Advanced Computation in KIAS for providing computing resources. RC thanks Michael Rozowski for helpful discussions on this project.

\section*{Impact Statement}

This paper presents work whose goal is to advance the field of Machine Learning. There are many potential societal consequences of our work, none which we feel must be specifically highlighted here.

\bibliography{main}

@InProceedings{sjko,
  title = 	 {Scalable {W}asserstein Gradient Flow for Generative Modeling through Unbalanced Optimal Transport},
  author =       {Choi, Jaemoo and Choi, Jaewoong and Kang, Myungjoo},
  booktitle = 	 {Proceedings of the 41st International Conference on Machine Learning},
  year = 	 {2024},
  volume = 	 {235},
  series = 	 {Proceedings of Machine Learning Research},
  month = 	 {21--27 Jul},
  publisher =    {PMLR},
}

@inproceedings{otm,
  title={Generative Modeling with Optimal Transport Maps},
  author={Rout, Litu and Korotin, Alexander and Burnaev, Evgeny},
  booktitle={International Conference on Learning Representations},
  year={2022}
}

@inproceedings{otmICNN,
  title={Optimal transport mapping via input convex neural networks},
  author={Makkuva, Ashok and Taghvaei, Amirhossein and Oh, Sewoong and Lee, Jason},
  booktitle={International Conference on Machine Learning},
  pages={6672--6681},
  year={2020},
  organization={PMLR}
}

@article{
    fanTMLR,
    title={Neural Monge Map estimation and its applications},
    author={Jiaojiao Fan and Shu Liu and Shaojun Ma and Hao-Min Zhou and Yongxin Chen},
    journal={Transactions on Machine Learning Research},
    issn={2835-8856},
    year={2023},
    url={https://openreview.net/forum?id=2mZSlQscj3},
    note={Featured Certification}
}

@article{Kantorovich1948,
  title={On a problem of Monge},
  author={Kantorovich, Leonid Vitalevich},
  journal={Uspekhi Mat. Nauk},
  pages={225--226},
  year={1948}
}

@inproceedings{fanscalable,
  title={Scalable Computation of Monge Maps with General Costs},
  author={Fan, Jiaojiao and Liu, Shu and Ma, Shaojun and Chen, Yongxin and Zhou, Hao-Min},
  booktitle={ICLR Workshop on Deep Generative Models for Highly Structured Data},
  year={2022},
}

@article{scoresde,
  title={Score-based generative modeling through stochastic differential equations},
  author={Song, Yang and Sohl-Dickstein, Jascha and Kingma, Diederik P and Kumar, Abhishek and Ermon, Stefano and Poole, Ben},
  journal={The International Conference on Learning Representations},
  year={2021}
}

@inproceedings{
    lipman2022flow,
    title={Flow Matching for Generative Modeling},
    author={Yaron Lipman and Ricky T. Q. Chen and Heli Ben-Hamu and Maximilian Nickel and Matthew Le},
    booktitle={The Eleventh International Conference on Learning Representations },
    year={2023},
    url={https://openreview.net/forum?id=PqvMRDCJT9t}
}

@article{santambrogio,
  title={Optimal transport for applied mathematicians},
  author={Santambrogio, Filippo},
  journal={Birk{\"a}user, NY},
  volume={55},
  number={58-63},
  pages={94},
  year={2015},
  publisher={Springer}
}

@InProceedings{choi2025overcoming,
  title = 	 {Overcoming Spurious Solutions in Semi-Dual Neural Optimal Transport: A Smoothing Approach for Learning the Optimal Transport Plan},
  author =       {Choi, Jaemoo and Choi, Jaewoong and Kwon, Dohyun},
  booktitle = 	 {Proceedings of the 42nd International Conference on Machine Learning},
  pages = 	 {10644--10665},
  year = 	 {2025},
  volume = 	 {267},
  series = 	 {Proceedings of Machine Learning Research},
  month = 	 {13--19 Jul},
  publisher =    {PMLR},
}

@inproceedings{uotm,
    title={Generative Modeling through the Semi-dual Formulation of Unbalanced Optimal Transport},
    author={Jaemoo Choi and Jaewoong Choi and Myungjoo Kang},
    booktitle={Thirty-seventh Conference on Neural Information Processing Systems},
    year={2023}
}

@inproceedings{which-converge,
  title={Which training methods for GANs do actually converge?},
  author={Mescheder, Lars and Geiger, Andreas and Nowozin, Sebastian},
  booktitle={International conference on machine learning},
  pages={3481--3490},
  year={2018},
  organization={PMLR}
}

@article{monge1781memoire,
  title={M{\'e}moire sur la th{\'e}orie des d{\'e}blais et des remblais},
  author={Monge, Gaspard},
  journal={Mem. Math. Phys. Acad. Royale Sci.},
  pages={666--704},
  year={1781}
}

@book{villani,
  title={Optimal transport: old and new},
  author={Villani, C{\'e}dric},
  volume={338},
  year={2009},
  publisher={Springer}
}

@inproceedings{not,
    title={Neural Optimal Transport},
    author={Alexander Korotin and Daniil Selikhanovych and Evgeny Burnaev},
    booktitle={The Eleventh International Conference on Learning Representations },
    year={2023}
}

@article{delalande2023quantitative,
  title={Quantitative stability of optimal transport maps under variations of the target measure},
  author={Delalande, Alex and Merigot, Quentin},
  journal={Duke Mathematical Journal},
  volume={172},
  number={17},
  pages={3321--3357},
  year={2023},
  publisher={Duke University Press}
}

@article{chewi2024statistical,
  title={Statistical optimal transport},
  author={Chewi, Sinho and Niles-Weed, Jonathan and Rigollet, Philippe},
  journal={arXiv preprint arXiv:2407.18163},
  year={2024}
}

@book{rockafellar1998variational,
  title={Variational analysis},
  author={Rockafellar, R Tyrrell and Wets, Roger JB},
  year={1998},
  publisher={Springer}
}

@article{block2021intrinsic,
  title={Intrinsic dimension estimation using wasserstein distances},
  author={Block, Adam and Jia, Zeyu and Polyanskiy, Yury and Rakhlin, Alexander},
  journal={arXiv preprint arXiv:2106.04018},
  year={2021}
}

@article{gigli2011holder,
  title={On H{\"o}lder continuity-in-time of the optimal transport map towards measures along a curve},
  author={Gigli, Nicola},
  journal={Proceedings of the Edinburgh Mathematical Society},
  volume={54},
  number={2},
  pages={401--409},
  year={2011},
  publisher={Cambridge University Press}
}

@inproceedings{merigot2020quantitative,
  title={Quantitative stability of optimal transport maps and linearization of the 2-Wasserstein space},
  author={M{\'e}rigot, Quentin and Delalande, Alex and Chazal, Frederic},
  booktitle={International Conference on Artificial Intelligence and Statistics},
  pages={3186--3196},
  year={2020},
  organization={PMLR}
}

@inproceedings{uot-upc,
  title={Unpaired Point Cloud Completion via Unbalanced Optimal Transport},
  author={Lee, Taekyung and Choi, Jaemoo and Choi, Jaewoong and Kang, Myungjoo},
  booktitle={Forty-second International Conference on Machine Learning}
}

@article{gazdieva2025optimal,
  title={An optimal transport perspective on unpaired image super-resolution},
  author={Gazdieva, Milena and Mokrov, Petr and Rout, Litu and Korotin, Alexander and Kravchenko, Andrey and Filippov, Alexander and Burnaev, Evgeny},
  journal={Journal of Optimization Theory and Applications},
  volume={207},
  number={2},
  pages={40},
  year={2025},
  publisher={Springer}
}
\bibliographystyle{./icml_style/icml2026}

\clearpage
\appendix
\onecolumn

\section{Proofs of Main Results}\label{app:proofs-main}

We now prove the statements from the main text.

\section{Proofs for Section~\ref{sec:semidual_failure}}\label{app:semidual_failure}

This appendix contains the proofs of the results stated in
Section~\ref{sec:semidual_failure}, which analyzes the recovery step
in the semi-dual formulation.

\subsection{Proof of Theorem~\ref{thm:recovery_characterization}}

\begin{proof}[Proof of Theorem~\ref{thm:recovery_characterization}]
Let $V\in C(\mathcal Y)$ and let $c\in C(\mathcal X\times\mathcal Y)$.
Recall that the $c$-transform of $V$ is defined by
\begin{equation}  V^c(x) := \inf_{y\in\mathcal Y} \bigl\{c(x,y)-V(y)\bigr\}. \end{equation}

\emph{($\Rightarrow$)}
Assume that a map $T_V:\mathcal X\to\mathcal Y$ satisfies the condition
\eqref{eq:recovery_constraint}. By definition of the $c$-transform, for every $x\in\mathcal X$ and every $y\in\mathcal Y$,
\[
V^c(x)\le c(x,y)-V(y).
\]

In particular, taking $y=T_V(x)$ yields
\begin{equation}
V^c(x)\le c(x,T_V(x)) - V(T_V(x)) \qquad \text{for all }x. \label{c_transform_ineq}
\end{equation}
Integrating against $\mu$ gives
\[
\int_{\mathcal X} V^c(x)\,d\mu(x)
\le
\int_{\mathcal X}\bigl(c(x,T_V(x)) - V(T_V(x))\bigr)\,d\mu(x).
\]
Since the right-hand side equals the minimal value of the recovery objective,
equality must hold. Combining this with \eqref{c_transform_ineq} implies
\[
c(x,T_V(x)) - V(T_V(x)) = V^c(x)
\quad\text{for $\mu$-a.e.\ }x.
\]

\smallskip
\noindent\emph{($\Leftarrow$)}
Conversely, suppose that a map $T_V:\mathcal X\to\mathcal Y$ satisfies
\[
c(x,T_V(x)) - V(T_V(x)) = V^c(x)
\quad\text{for $\mu$-a.e.\ }x.
\]
Then for any measurable map $T:\mathcal X\to\mathcal Y$, the defining inequality of the
$c$-transform implies
\[
V^c(x)\le c(x,T(x)) - V(T(x)) \quad\text{for all }x.
\]
Integrating with respect to $\mu$ yields
\[
\int_{\mathcal X} V^c(x)\,d\mu(x)
\le
\int_{\mathcal X}\bigl(c(x,T(x)) - V(T(x))\bigr)\,d\mu(x),
\]
with equality attained when $T=T_V$.
Hence $T_V$ minimizes the recovery objective and satisfies the recovery condition
\eqref{eq:recovery_constraint}.

\smallskip
\noindent
Finally, since $c(x,y)-V(y)$ is continuous in $y$ and $\mathcal Y$ is compact,
the infimum defining $V^c(x)$ is attained for every $x$.
\end{proof}

\subsection{Proof of Theorem~\ref{thm:tangential_recovery}} To prove Theorem~\ref{thm:tangential_recovery}, we will first need the following tangential touching lemma:

\begin{lemma}[Tangential gradient at a contact point] \label{tangential_contact_point}
Let $\mathcal M \subset \R^d$ be a smooth embedded Riemannian manifold.
Let $f \in C^1(\R^d)$ and $\varphi \in C^1(\mathcal M)$ satisfy
\[
\varphi \le f \quad \text{on } \mathcal M,
\qquad
\varphi(x^*) = f(x^*) \quad \text{for some } x^* \in \mathcal M.
\]
Then
\[
\nabla_{\mathcal M} \varphi(x^*)
=
\Pi_{T_{x^*}\mathcal M} \nabla f(x^*).
\]
\end{lemma}

    \begin{proof}
Let $v \in T_{x^*}\mathcal M$ be arbitrary, and let
$\gamma : (-\varepsilon,\varepsilon) \to \mathcal M$ be a smooth curve
such that $\gamma(0)=x^*$ and $\dot\gamma(0)=v$.
Since $\varphi \le f$ on $\mathcal M$ and $\varphi(x^*)=f(x^*)$,
the function
\[
t \mapsto (\varphi - f)(\gamma(t))
\]
attains a maximum at $t=0$. Therefore,
\[
0
=
\frac{d}{dt}\Big|_{t=0} (\varphi - f)(\gamma(t))
=
\langle \nabla_{\mathcal M} \varphi(x^*), v \rangle
-
\langle \nabla f(x^*), v \rangle.
\]
Since $v \in T_{x^*}\mathcal M$ was arbitrary, this implies the desired result.
\end{proof}

Now we can prove Theorem~\ref{thm:tangential_recovery}.

    \begin{proof}[Proof of Theorem~\ref{thm:tangential_recovery}]
Let $(V^*,T_{\mathrm{rec}})$ be an optimal solution of~(1).
By the characterization of the recovery step (Theorem~\ref{thm:recovery_characterization}),
the recovered map $T_{\mathrm{rec}}$ satisfies, for $\mu$-a.e.\ $x\in\mathcal M$,
\begin{equation}\label{eq:rec_attains_c_transform}
c\bigl(x,T_{\mathrm{rec}}(x)\bigr)-V^*\bigl(T_{\mathrm{rec}}(x)\bigr)
=
(V^*)^c(x).
\end{equation}

Since $(V^*,T_{\mathrm{rec}})$ is optimal for~(1), taking the infimum over $T$ yields
\[
\inf_T \mathcal L(V^*,T)
=
\int_{\mathcal X} (V^*)^c(x)\,d\mu(x) + \int_{\mathcal Y} V^*(y)\,d\nu(y).
\]
Moreover, maximizing over $V$ shows that $V^*$ attains the semi-dual optimum, hence from equality of dual and primal problem (see \citet{santambrogio})
\[
\int_{\mathcal X} (V^*)^c(x)\,d\mu(x) + \int_{\mathcal Y} V^*(y)\,d\nu(y)
=
\inf_{\pi\in\Pi(\mu,\nu)} \iint_{\mathcal X\times\mathcal Y} c(x,y)\,d\pi(x,y),
\]
so $((V^*)^c,V^*)$ is a pair of Kantorovich potentials for the cost $c$.

Now assume that an optimal transport map $T_{\mathrm{opt}}$ exists between $\mu$ and $\nu$,
and set $\pi_{\mathrm{opt}}:=(\mathrm{Id},T_{\mathrm{opt}})_\#\mu\in\Pi(\mu,\nu)$.
By definition of the $c$-transform we have the pointwise inequality
\begin{equation}\label{eq:dual_ineq}
(V^*)^c(x) + V^*(y)\le c(x,y)
\qquad \forall\,x\in\mathcal X,\ \forall\,y\in\mathcal Y.
\end{equation}
Integrating \eqref{eq:dual_ineq} against $\pi_{\mathrm{opt}}$ and using
$\pi_{\mathrm{opt}}$ has marginals $(\mu,\nu)$ yields
\[
\int_{\mathcal X} (V^*)^c\,d\mu + \int_{\mathcal Y} V^*\,d\nu
\le
\int_{\mathcal X} c\bigl(x,T_{\mathrm{opt}}(x)\bigr)\,d\mu(x).
\]
But the left-hand side equals the optimal transport cost (from equality of the primal and dual problem), and the right-hand side
also equals the optimal cost since $T_{\mathrm{opt}}$ is optimal. Hence equality holds:
\[
\int_{\mathcal X} \Big(c\bigl(x,T_{\mathrm{opt}}(x)\bigr) - (V^*)^c(x) - V^*\bigl(T_{\mathrm{opt}}(x)\bigr)\Big)\,d\mu(x)=0.
\]
The integrand is nonnegative by \eqref{eq:dual_ineq}, so it must vanish $\mu$-a.e., i.e.
\begin{equation}\label{eq:opt_attains_c_transform}
c\bigl(x,T_{\mathrm{opt}}(x)\bigr) - V^*\bigl(T_{\mathrm{opt}}(x)\bigr)
=
(V^*)^c(x)
\qquad\text{for $\mu$-a.e.\ }x\in\mathcal M.
\end{equation}

Since $c\in C^1(\mathcal X\times\mathcal Y)$ and both $\mathcal X$ and $\mathcal Y$ are compact,
the map $x\mapsto c(x,y)$ is Lipschitz continuous on $\mathcal X$, with a Lipschitz constant
uniform in $y$.
As a consequence, the $c$-transform $(V^*)^c$, being the infimum of a family of uniformly Lipschitz
functions, is itself Lipschitz continuous on $\mathcal X$ with respect to the Euclidean metric.

Let $d_{\mathcal M}$ denote the intrinsic geodesic distance on $\mathcal M$.
Since $\mathcal M$ is embedded in $\mathbb R^d$, the Euclidean distance is dominated by the geodesic distance,
i.e.\ $|x-x'|\le d_{\mathcal M}(x,x')$ for all $x,x'\in\mathcal M$.
Therefore, the restriction $(V^*)^c|_{\mathcal M}$ is Lipschitz continuous with respect to $d_{\mathcal M}$.

It follows from Rademacher theorem that $(V^*)^c|_{\cM}$ is  differentiable a.e. on $\mathcal M$ with respect to $\mathcal{H}^m \llcorner \cM$. Since $\mu \ll \mathcal{H}^m \llcorner \cM$, this implies that $(V^*)^c$ is tangentially differentiable
for $\mu$-a.e.\ $x\in\mathcal M$.

Fix $x\in\mathcal M$ in the full-measure set where \eqref{eq:rec_attains_c_transform},
\eqref{eq:opt_attains_c_transform}, and $(V^*)^c$ being tangentially differentiable holds.  Then \eqref{eq:opt_attains_c_transform}, \eqref{eq:dual_ineq}, and $(V^*)^c$ being tangentially differentiable implies from  Lemma~\ref{tangential_contact_point} with $\varphi_y(x) := V^*(y)+(V^*)^c(x)$ and $f_y(x) := c(x,y)$ implies at $y=T_{\mathrm{opt}}(x)$
\[  \nabla_{\cM} (V^*)^c(x) = \Pi_{T_x \cM} \nabla_x c(x,T_{\mathrm{opt}}(x)).   \] Similarly, because of \eqref{eq:rec_attains_c_transform} and \eqref{eq:dual_ineq} implies that
\[  \nabla_{\cM} (V^*)^c(x) = \Pi_{T_x \cM} \nabla_x c(x,T_{\mathrm{rec}}(x)).   \] This implies the desired result.
\end{proof}

\section{Proofs of Section \ref{section:regularization}}

We begin with the full recovery Lemma mentioned in Section~\ref{section:regularization} when the source measure is full dimensional.

\begin{lemma}[Full recovery in the full-dimensional case]\label{lem:full_recovery}
Let $K\subset\R^d$ be compact with $\mu(\partial K)=0$, $\mu \in \cP(K)$, and $\mu \ll \cL^d$. Assume the cost has the form $c(x,y)=h(x-y)$, where $h\in C^1(\R^d)$ is strictly convex.
Let $(V^*,T_{\mathrm{rec}})$ be an optimal solution of~\eqref{eq:snot_intro}, and then if
$T_{\mathrm{opt}}$ denotes the unique optimal transport map from $\mu$ to $\nu$.
Then
\[
T_{\mathrm{rec}}(x)=T_{\mathrm{opt}}(x)
\qquad\text{for }\mu\text{-a.e.\ }x.
\]
\end{lemma}

\begin{proof}[Proof of Lemma~\ref{lem:full_recovery}] From our assumptions, we can apply \citep[Theorem 1.17]{santambrogio} to conclude there exists a unique optimal transport map. \smallskip
 
Set $\mathcal M:=\operatorname{int}(K)$.
Since $\mu(\partial K)=0$ and $\mu$ is supported on $K$, we have $\mu(\R^d\setminus\mathcal M)=0$.
In particular, $\mu$-a.e.\ $x$ lies in $\mathcal M$, and for such $x$ we have $T_x\mathcal M=\R^d$.

Apply Theorem~\ref{thm:tangential_recovery} with $\mathcal M=\operatorname{int}(K)$.
Since $T_x\mathcal M=\R^d$ for $\mu$-a.e.\ $x$, the orthogonal projection
$\Pi_{T_x\mathcal M}$ is the identity, and we obtain
\begin{equation}\label{eq:fullgrad_equal}
\nabla_x c\bigl(x,T_{\mathrm{rec}}(x)\bigr)
=
\nabla_x c\bigl(x,T_{\mathrm{opt}}(x)\bigr)
\qquad\text{for }\mu\text{-a.e.\ }x\in \mathcal M.
\end{equation}

Now use the special form $c(x,y)=h(x-y)$ with $h\in C^1(\R^d)$ strictly convex.
For each fixed $x$, we have
\[
\nabla_x c(x,y)=\nabla h(x-y).
\]
Since $h$ is strictly convex, its gradient $\nabla h:\R^d\to\R^d$ is injective.
Thus \eqref{eq:fullgrad_equal} implies
\[
\nabla h\bigl(x-T_{\mathrm{rec}}(x)\bigr)=\nabla h\bigl(x-T_{\mathrm{opt}}(x)\bigr)
\quad\Longrightarrow\quad
x-T_{\mathrm{rec}}(x)=x-T_{\mathrm{opt}}(x),
\]
hence $T_{\mathrm{rec}}(x)=T_{\mathrm{opt}}(x)$ for $\mu$-a.e.\ $x$, which proves the lemma.
\end{proof}


\subsection{Proof of Theorem~\ref{thm:map_level_convg}}

Our first step is to prove a graphical convergence lemma.  We recall for a convex function $\psi$ that $\partial \psi(x_0)$ denotes the sub-differential at $x_0$ and we refer to \citet{rockafellar1998variational} for the definition of graphical convergence.

\begin{lemma}[Graphical convergence of regularized transport sets]\label{lem:graphical_convg} Let $K\subset\R^d$ be compact, $\nu\in\mathcal P_2(K)$, and $c(x,y)=\tfrac12|x-y|^2$. Let $\mu_\eps\in\mathcal P(K)$ satisfy $\mu_\eps\ll\mathcal L^d$, and $\mu_\eps\rightharpoonup\mu$ as $\eps\downarrow0$. For each $\eps>0$, let $(V_\eps,T_\eps)$ be optimal for \eqref{eq:snot_intro}. Then $(V_\eps^{cc},T_\eps)$ is also optimal. Moreover, along some $\eps_k\downarrow0$, \[ V_{\eps_k}^{cc}\to V_0,\qquad V_{\eps_k}^{c}\to V_0^{c}\qquad\text{uniformly on }K, \] and with $\phi_\eps:=\tfrac12|x|^2-V_\eps^c$ and $\phi_0:=\tfrac12|x|^2-V_0^c$, \[ \partial\phi_{\eps_k}\to\partial\phi_0\quad\text{graphically}. \] Finally, for any optimal plan $\pi^*$ between $\mu$ and $\nu$, one has $y_0\in\partial\phi_0(x_0)$ for $\pi^*$-a.e.\ $(x_0,y_0)$. \end{lemma}

\begin{proof}
\textbf{Step 1: Optimality of $(V_\eps^{cc},T_\eps)$.}
By Theorem~\ref{thm:recovery_characterization},
\[
\inf_{T}\cL(V,T)
= \int_K V^{c}\,d\mu_\eps + \int_{K} V\,d\nu .
\]
Since $(V_\eps,T_\eps)$ is optimal,
\begin{equation}\label{eq:val_Veps}
\cL(V_\eps,T_\eps)
= \int_K V_\eps^{c}\,d\mu_\eps + \int_{K} V_\eps\,d\nu .
\end{equation}

By \citet[Proposition 1.34]{santambrogio} we have
$V_\eps^{cc}\ge V_\eps$ pointwise and $(V_\eps^{cc})^{c}=V_\eps^{c}$ since $c(x,y)=c(y,x)$. Hence,
\[
\inf_T\cL(V_\eps^{cc},T)
= \int_K V_\eps^{c}\,d\mu_\eps + \int_{K} V_\eps^{cc}\,d\nu
\ge \int_K V_\eps^{c}\,d\mu_\eps + \int_{K} V_\eps\,d\nu
= \cL(V_\eps,T_\eps).
\]
Since $(V_\eps,T_\eps)$ is optimal, equality must hold, and therefore
$(V_\eps^{cc},T_\eps)$ is also optimal.

\medskip
\textbf{Step 2: Uniform convergence.}
The $c$-transform is $1$-Lipschitz in the supremum norm:
\begin{equation}\label{eq:c_lipschitz}
\|f^{c}-g^{c}\|_{L^\infty(K)}\le \|f-g\|_{L^\infty(K)}.
\end{equation} Indeed, let $\delta:=\|f-g\|_{L^\infty(K)}$. Then $g(y)\le f(y)+\delta$ for all $y\in K$, hence
\[
c(x,y)-g(y)\ge c(x,y)-f(y)-\delta \geq f^c(x) - \delta \qquad \forall x,y\in K .
\]
Taking $\inf_{y\in K}$ yields $g^{c}(x)\ge f^{c}(x)-\delta$, i.e.\ $f^{c}(x)-g^{c}(x)\le \delta$.
Exchanging $f$ and $g$ gives $g^{c}(x)-f^{c}(x)\le \delta$. Therefore
\[
|f^{c}(x)-g^{c}(x)|\le \delta \quad \forall x\in K,
\] which implies \eqref{eq:c_lipschitz}. Now \eqref{eq:c_lipschitz} implies that uniform convergence of $V_{\eps}^{c}$ implies uniform convergence of $V_{\eps}^{cc}$.

Since $c(x,y)=\tfrac12|x-y|^2$ is uniformly Lipschitz on $K\times K$,
the functions $V_\eps^{c}$ are uniformly Lipschitz on $K$.
After normalization (e.g.\ $V_\eps^{c}(0)=0$), they are also uniformly bounded.
By Arzelà–Ascoli and \eqref{eq:c_lipschitz}, there exists $\eps_k\downarrow0$ such that
\[
V_{\eps_k}^{c}\to V_0^{c}
\quad\text{and}\quad
V_{\eps_k}^{cc}\to V_0
\quad\text{uniformly on }K.
\]

\medskip
\textbf{Step 3: Graphical convergence.}
Define $\phi_\eps:=\tfrac12|x|^2 - V_\eps^{c}$ and
$\phi_0:=\tfrac12|x|^2 - V_0^{c}$.
Each $\phi_\eps$ is convex, and $\phi_{\eps_k}\to\phi_0$ uniformly on $K$.
By \citep[Theorem~12.35]{rockafellar1998variational},
\[
\partial\phi_{\eps_k}\to\partial\phi_0
\quad\text{graphically}.
\]

\textbf{Step 4: Support characterization.}
For all $x,y\in K$,
\[
V_\eps^{c}(x)+V_\eps^{cc}(y)\le \tfrac12|x-y|^2.
\]
Passing to the limit yields for all $x,y \in K$
\begin{equation}
V_0^{c}(x)+V_0(y)\le \tfrac12|x-y|^2. \label{bound_on_Kantrovich_potential}
\end{equation} Now we observe that
\[ \inf_T \mathcal{L}(V_{\eps}^{cc},T) = \sup_{V } \inf_T \mathcal{L}(V,T) = W_2^2(\mu_{\eps},\nu),  \] so from \citep[Theorem 1.51]{santambrogio} and letting $\varepsilon \downarrow 0$, we obtain that
\[ \int_K V_0^c d\mu(x) + \int_K V_0 d\nu(x) = W_2^2(\mu,\nu) = \iint_{K \times K} \tfrac 12 |x-y|^2 d\pi^*(x,y), \] Hence, \eqref{bound_on_Kantrovich_potential} and the above equality implies that $V_0^c(x)+V_0(y) = \tfrac 12 |x-y|^2$ for $\pi^*$-a.e.\ $(x,y)\in\supp(\pi^*)$.
Equivalently,
\[
\phi_0(x)-\phi_0(x_0)\ge \langle x-x_0,y_0\rangle
\quad\text{for all }x,
\]
which implies $y_0\in\partial\phi_0(x_0)$ for $\pi^*$-a.e.\ $(x_0,y_0)$.
\end{proof}

Now we can prove Theorem~\ref{thm:map_level_convg}.

\begin{proof}[Proof of Theorem~\ref{thm:map_level_convg}] By Lemma~\ref{lem:graphical_convg}, along a subsequence $\eps_k\downarrow0$ we have
uniform convergence $V_{\eps_k}^{cc}\to V_0$ and $V_{\eps_k}^c\to V_0^c$ on $K$, and graphical convergence
$\partial\phi_{\eps_k}\to\partial\phi_0$, where
\[
\phi_\eps(x):=\tfrac12|x|^2 - V_\eps^c(x),
\qquad
\phi_0(x):=\tfrac12|x|^2 - V_0^c(x).
\]

\textbf{Proof of (1).}
The function $\phi_0$ is convex. By Lemma~\ref{lem:graphical_convg}, for any optimal plan $\pi^*$ between
$\mu$ and $\nu$,
\[
y\in \partial\phi_0(x)\qquad\text{for }\pi^*\text{-a.e. }(x,y).
\]
If $V_0^c$ is differentiable $\mu$-a.e., then $\phi_0$ is differentiable $\mu$-a.e. and
\[
\partial\phi_0(x)=\{\nabla\phi_0(x)\}=\{x-\nabla V_0^c(x)\}
\qquad\text{for }\mu\text{-a.e. }x.
\]
Define $T_0(x):=x-\nabla V_0^c(x)$. The support condition above implies that any optimal plan is concentrated
on the graph of $T_0$, hence $\pi^*=(\Id,T_0)_\#\mu$ and $T_{0\#}\mu=\nu$. Therefore $T_0$ is the optimal
Monge map from $\mu$ to $\nu$.

\medskip
\textbf{Proof of (2).}
Fix $x\in K$ at which $V_0^c$ is differentiable (hence $\phi_0$ is differentiable). Then
$\partial\phi_0(x)=\{\nabla\phi_0(x)\}$ with $\nabla\phi_0(x)=x-\nabla V_0^c(x)$.
By graphical convergence and \citep[Exercise~4.2(a)]{rockafellar1998variational}, there exist $x_k\to x$ and
$p_k\in\partial\phi_{\eps_k}(x_k)$ such that $p_k\to\nabla\phi_0(x)$.
Since each $\phi_{\eps_k}$ is differentiable, $p_k=\nabla\phi_{\eps_k}(x_k)=x_k-\nabla V_{\eps_k}^c(x_k)$,
and therefore
\[
\nabla V_{\eps_k}^c(x_k)\to \nabla V_0^c(x),
\qquad\text{hence}\qquad
T_{\eps_k}(x_k)=x_k-\nabla V_{\eps_k}^c(x_k)\to x-\nabla V_0^c(x)=T_0(x).
\]

\medskip
\textbf{Optimality of $(V_0,T_0)$ for $(\mu,\nu)$.}
For each $\eps$, $(V_\eps^{cc},T_\eps)$ is optimal for \eqref{eq:snot_intro} with measures $(\mu_\eps,\nu)$,
so by Theorem~\ref{thm:recovery_characterization},
\[
\cL(V_\eps^{cc},T_\eps)=\int_K V_\eps^c\,d\mu_\eps+\int_K V_\eps^{cc}\,d\nu
= \sup_V \int_K V^c d\mu + \int_K V d\nu = W_2^2(\mu_\eps,\nu).
\]
Stability of $W_2$ under weak convergence (see \citep[Theorem 1.51]{santambrogio}) yields
$W_2^2(\mu_{\eps_k},\nu)\to W_2^2(\mu,\nu)$.
Using the uniform convergence $V_{\eps_k}^c\to V_0^c$ and $V_{\eps_k}^{cc}\to V_0$ on $K$ and
$\mu_{\eps_k}\rightharpoonup\mu$, we obtain
\[
\int_K V_{\eps_k}^c\,d\mu_{\eps_k}+\int_K V_{\eps_k}^{cc}\,d\nu
\;\longrightarrow\;
\int_K V_0^c\,d\mu+\int_K V_0\,d\nu.
\]
Hence,
\begin{equation}
\int_K V_0^c\,d\mu+\int_K V_0\,d\nu \;=\; W_2^2(\mu,\nu), \label{equality_primal_dual}
\end{equation}
so $V_0$ is a dual optimizer for $(\mu,\nu)$. To conclude that $(V_0,T_0)$ is an optimal solution for \eqref{eq:snot_intro}, by Theorem~\ref{thm:recovery_characterization} it remains to show that
\begin{equation} V_0^c(x) =  c(x,T_0(x))-V(T_0(x)) \text{ for } \mu \text{ a.e. } x. \label{equality_support_VC} \end{equation} Observe from the definition of the c-transform, we always have that $ V_0^c(x) \leq  c(x,T_0(x))-V(T_0(x))$, so it suffices to show the reverse inequality. Now from \eqref{equality_primal_dual} and the existence of an optimal map we have 
\[ \int_{\cX} V_0^c(x) d\mu(x) + \int_{\cY} V_0(y) d\nu(y)  = \cL(V_{\eps}^{cc},T_{\eps}) = W_2(\mu,\nu) = \int_{\cX} c(x,T_0(x)) d\mu(x).  \] This combined with $\nu = (T_0)_{\#} \mu$ and $V_0^c(x) \leq  c(x,T_0(x))-V(T_0(x))$ implies \eqref{equality_support_VC}.
\end{proof}

\subsubsection{Proof of Remark~\ref{rmk_c_transform}} \label{proof_remark}

\begin{proof}[Proof of Remark~\ref{rmk_c_transform}] Observe that as $(V_{\eps},T_{\eps})$ is an optimal solution of \eqref{eq:snot_intro} with source measure $\mu_{\eps}$ and target measure $\nu$, we have
\[ W_2^2(\mu_{\eps},\nu) = \int_{K} V_{\eps}^c(x) d\mu(x) + \int_K V_{\eps}(x) d\nu(x).  \] Because $\nu \ll \mathcal{L}^d$ and $\mathcal{L}^d(\partial K)=0$, \citep[Theorem 1.17]{santambrogio} implies there is an optimal transport map $T_{\#} \nu = \mu_{\eps}$ such that
\[ \int_{K} V_{\eps}^c(x) d\mu(x) + \int_K V_{\eps}(x) d\nu(x) = W_2^2(\mu_{\eps},\nu) = \int_K \tfrac12 |x-T(x)|^2 d\nu(x). \] By the definition of the c-transform, we have $V_{\eps}^c(x) + V_{\eps}(y) \leq \tfrac 12 |x-y|^2$ for all $x,y$. Hence, the above equality implies $V_{\eps}^c(y) + V_{\eps}(x) = \tfrac 12 |x-y|^2$ for $y=T(x)$. Hence, from $f_y(z) := V_{\eps}^c(y) + V_{\eps}(z) - \tfrac 12 |z-y|^2$ for $y=T(x)$ obtains a global maximum at $z=x$. Therefore, from $V_{\eps}$ being Lipschitz, it is differentiable Lebesgue a.e. by Rademacher's theorem. Hence, we have that for Lebesgue a.e. $x$
\[ \nabla V_{\eps}(x) =  x-T(x). \] Recalling the arguments of Lemma~\ref{lem:graphical_convg} implies
\[  \int_{K} V_{\eps}^c(x) d\mu(x) + \int_K V_{\eps}(x) d\nu(x) = \int_{K} V_{\eps}^c(x) d\mu(x) + \int_K V^{cc}_{\eps}(x) d\nu(x).  \] Now as $V^{cc}$ is Lipschitz continuous due to the c-transform, we can repeat the argument above implies we have that for Lebesgue a.e. $x$
\[ \nabla V_{\eps}^{cc}(x) = x - T(x) \]
\end{proof}

\subsubsection{Example~\ref{maps_converge_wrongly}}

We give an example where the pointwise limit of $T_\varepsilon$ exists, but the limit is not a transport map. This shows that the differentiability assumption on the limiting potential in Theorem~\ref{thm:map_level_convg} is essential.

\begin{example}\label{maps_converge_wrongly}
Let $d=1$, $\mu=\delta_0$, and $\nu=\textnormal{Unif}(0,1)$. With $Y\sim\mathcal{N}(0,1)$, we have
$\mu_\varepsilon=\operatorname{Law}(\varepsilon Y)\sim\mathcal{N}(0,\varepsilon^2)$. As $d=1$, we find that the optimal transport map from $\mu_\varepsilon$ to $\nu$ is
\[
T^\varepsilon(x)=\Phi(x/\varepsilon),
\]
where $\Phi$ is the CDF of $\mathcal{N}(0,1)$. As $\varepsilon\to 0^+$,
\[
T^\varepsilon(x)\to T^0(x):=
\begin{cases}
0, & x<0,\\
\frac12, & x=0,\\
1, & x>0.
\end{cases}
\] Thus $(T^0)_{\#}\mu=\delta_{1/2}\neq \nu$, so the pointwise limit is not a transport map.

The associated Kantrovich potential (up to constants) is given by $V^c_\varepsilon(x)=\int_{-\infty}^x T^\varepsilon(t)\,dt$. One has that $(V^{cc}_{\eps},T^{\eps})$ is an optimal solution of \eqref{eq:snot_intro}. However,
\[  V^c_\varepsilon(x)\to x_+:=\max\{x,0\} \] This limiting potential is not differentiable at
$x=0=\operatorname{supp}(\mu)$, which shows the differentiability assumption of Theorem~\ref{thm:map_level_convg} is essentinal. 
\end{example}

\subsubsection{Proof of Proposition~\ref{thm:noise_tangential}}

\begin{proof}[Proof of Proposition~\ref{thm:noise_tangential}] By Lemma~\ref{lem:graphical_convg} we have a limiting potential.

 Let $V_0$ be a limiting potential. We will first show that $V_0$ maximizes \eqref{eq:opt_pair_V}. First the arguments of Theorem~\ref{thm:map_level_convg} implies that
 \begin{equation} \int_{\cX} V_0^c(x) d\mu(x) + \int_{\cY} V_0(y)d\nu(y) = W_2(\mu,\nu). \label{equality_primal_dual_123214}\end{equation} Now as
 \[ \inf_T\cL(V,T) = \int V^c d\mu(x) + \int V d\nu(x) \leq W_2(\mu,\nu),  \] we conclude that $V_0$ is an optimizer for \eqref{eq:opt_pair_V}. \smallskip

 Now for the tangential identity, we observe from the definition of the c-transform that
\begin{equation}\label{eq:dual_ineq_prop43}
V_0^c(x)+V_0(y)\le c(x,y)\qquad\text{for all }(x,y)\in \R^d\times\R^d,
\end{equation}
and suppose an optimal Monge map $T_{\operatorname{opt}}$ from $\mu$ to $\nu$ exists. Then \eqref{equality_primal_dual_123214} and the existence of the optimal map implies
\begin{equation}\label{eq:slackness_prop43}
V_0^c(x)+V_0(T_{\operatorname{opt}}(x))=c\bigl(x,T_{\operatorname{opt}}(x)\bigr)\qquad\text{for }\mu\text{-a.e. }x.
\end{equation}

\smallskip
We now show that $V_0^c$ is tangentially differentiable. Since $c(\cdot,y)$ is $C^1$ and $\supp(\nu)$ is compact, $x\mapsto c(x,y)$ is uniformly Lipschitz on $K$
for $y\in\supp(\nu)$, and hence $V_0^c(x)=\inf_y\{c(x,y)-V_0(y)\}$ is Lipschitz on $K$.
In particular, $V_0^c|_{\cM}$ is Lipschitz with respect to the geodesic distance on $\cM$ because the Euclidean distance is dominated by the geodesic distance on $\cM$.
By Rademacher's theorem $V_0^c|_{\cM}$ is  differentiable
$\mathcal H^m\llcorner \cM$-a.e.; equivalently, $\nabla_{\cM}V_0^c(x)$ exists for $\mu$-a.e.\ $x\in\cM$.

\smallskip
Fix $x_0\in\cM$ such that $V_0^c$ is tangentially differentiable at $x_0$ and
\eqref{eq:slackness_prop43} holds at $x_0$, and set $y_0:=T_{\operatorname{opt}}(x_0)$.
Define the smooth function
\[
f_{y_0}(x):=c(x,y_0)-V_0(y_0).
\]
By the definition of the $c$-transform,
\[
V_0^c(x)=\inf_{y}\{c(x,y)-V_0(y)\}\le c(x,y_0)-V_0(y_0)=f_{y_0}(x)
\qquad\text{for all }x\in\cM,
\]
and by \eqref{eq:slackness_prop43} we have equality at $x_0$:
\[
V_0^c(x_0)=c(x_0,y_0)-V_0(y_0)=f_{y_0}(x_0).
\]
Hence the hypotheses of Lemma~\eqref{tangential_contact_point} apply (with $\varphi=V_0^c|_{\cM}$ and $f=f_{y_0}$), yielding
\[
\nabla_{\cM}V_0^c(x_0)
=\Pi_{T_{x_0}\cM}\nabla f_{y_0}(x_0)
=\Pi_{T_{x_0}\cM}\nabla_x c(x_0,y_0)
=\Pi_{T_{x_0}\cM}\nabla_x c\bigl(x_0,T_{\operatorname{opt}}(x_0)\bigr).
\]
Since $x_0$ was arbitrary outside a $\mu$-null set, this proves
\[
\nabla_{\cM}V_0^c(x)=\Pi_{T_x\cM}\nabla_x c\bigl(x,T_{\operatorname{opt}}(x)\bigr)
\qquad\text{for }\mu\text{-a.e. }x\in\cM.
\] \end{proof}

\section{Proofs of Section~\ref{sec:noise_schedule}}

We begin by proving a general quantitative stability estimate on the optimal plans. This result is obtained from the stability estimates of \citep{merigot2020quantitative}.

\begin{theorem}[Quantitative stability of plans]\label{thm:quant_general}
Fix $\nu\in\mathcal P_2(\R^d)$ and assume:
\begin{enumerate}[label=(\roman*),leftmargin=2.2em]
\item $\nu\ll\mathcal L^d$,
\item $\supp(\nu)$ is compact and convex,
\item there exists $\Lambda>0$ such that $\Lambda^{-1}\le \nu \le \Lambda$ on $\supp(\nu)$ (a.e.).
\end{enumerate}
Fix $p\ge \max(4,d)$ and define $\alpha(p)=\frac{p}{6p+16d}$.
Let $\mu_1,\mu_2 \in\mathcal P_2(\R^d)$ be such that $\int |x|^p d\mu_i \leq M$ for some $M>0$. Then one has
\[
W_2(\pi(\mu_1,\nu),\pi(\mu_2,\nu)) \lesssim_{d,p,\nu,\Lambda,M} W_1(\mu_1,\mu_2)^{\alpha(p)}.
\] 
\end{theorem}

\begin{proof}
Under the assumptions on $\nu$, \citet{delalande2023quantitative} gives
\begin{equation}
\|T_{\nu\to\mu_1}-T_{\nu\to\mu_2}\|_{L^2(\nu)}
\le
C\,\big(W_1(\mu_1,\mu_2)\big)^{\alpha(p)}, \label{map_stability_est}
\end{equation}
where $T_{\nu\to\rho}$ denotes the optimal transport map pushing $\nu$ forward to $\rho$.
By Brenier's theorem,
\[
\pi(\mu_i,\nu)=(T_{\nu\to\mu_i},\mathrm{Id})_\#\nu,
\qquad i=1,2.
\]
Define a coupling $\Gamma$ between $\pi(\mu_1,\nu)$ and $\pi(\mu_2,\nu)$ by
\[
\Gamma := \bigl((T_{\nu\to\mu_1},\mathrm{Id}),
(T_{\nu\to\mu_2},\mathrm{Id})\bigr)_\#\nu .
\]
Then $\Gamma$ has marginals $\pi(\mu_1,\nu)$ and $\pi(\mu_2,\nu)$. Therefore,
\[
\begin{aligned}
W_2^2(\pi(\mu_1,\nu),\pi(\mu_2,\nu))
&\le \int \bigl|(T_{\nu\to\mu_1}(y),y)
-(T_{\nu\to\mu_2}(y),y)\bigr|^2\,d\nu(y) \\
&= \int |T_{\nu\to\mu_1}(y)-T_{\nu\to\mu_2}(y)|^2\,d\nu(y),
\end{aligned}
\]
which yields from \eqref{map_stability_est}
\[
W_2(\pi(\mu_1,\nu),\pi(\mu_2,\nu))
\le \|T_{\nu\to\mu_1}-T_{\nu\to\mu_2}\|_{L^2(\nu)} \lesssim_{d,p,\nu,\Lambda,M} W_1(\mu,\mu_n)^{\alpha(p)} .
\]
\end{proof}

\subsection{Proof of Theorem~\ref{thm:semi_discrete}}

Now we prove Theorem \ref{thm:semi_discrete}:

\begin{proof}[Proof of Theorem~\ref{thm:semi_discrete}]
By Theorem~\ref{thm:quant_general}, it suffices to show that
\begin{equation}
\mathbb{E}\, W_1(\mu,\mu_N^{\varepsilon})
\;\lesssim\;
\varepsilon \;+\;
\begin{cases}
N^{-1/2}, & m=1,\\[2pt]
(\log N/N)^{1/2}, & m=2,\\[2pt]
N^{-1/m}, & m\ge 3.
\end{cases} \label{desired_bound_123}
\end{equation} 
To see this we let $\mu_N := \frac{1}{N} \sum_{j=1}^{N} \delta_{X_j}$ and observe the triangle inequality implies:
\[
\mathbb{E}W_1(\mu,\mu_N^{\varepsilon})
\le \mathbb{E}W_1(\mu,\mu_{N}) + \mathbb{E}W_1(\mu_{N},\mu_N^{\varepsilon})
\lesssim \mathbb{E} W_1(\mu_N,\mu_N^{\eps}) + \begin{cases}
N^{-1/2}, & m=1,\\[2pt]
(\log N/N)^{1/2}, & m=2,\\[2pt]
N^{-1/m}, & m\ge 3.
\end{cases}, \] where in the above inequality we used that $\mu$ is supported on a manifold of dimension $m$ to apply \citep[Lemma 15]{block2021intrinsic}. 

 To bound the remaining term,  we use the Kantorovich--Rubinstein duality. For any $1$-Lipschitz
function $\varphi$ we have
\[
\left| \int \varphi\, d\mu_N - \int \varphi\, d\mu_N^{\varepsilon}
 \right| \leq \frac{1}{N}\sum_{j=1}^{N} \left| \big(\varphi(X_j)-\varphi(X_j+\varepsilon Y_j)\big) \right| \leq \frac{\varepsilon}{N}\sum_{j=1}^{N}|Y_j|.
\]Taking the supremum over such $\varphi$ yields
\[
\mathbb{E}W_1(\mu_N,\mu_N^{\varepsilon})
\le \eps \mathbb{E}[|Y|].
\] This implies \eqref{desired_bound_123}, which allows us to conclude from Theorem~\ref{thm:quant_general}.
\end{proof}

\subsection{Proof of Theorem~\ref{thm:lower_bounds}}

We now prove a lower bound on learning plans.

\begin{theorem}[Lower bounds on finite-sample stability with noise]\label{thm:lower_bounds}
With the notation of Theorem~\ref{thm:semi_discrete}, for any $\nu \in \cP_2(\R^d)$,
\[
\mathbb{E}\Big[ W_2\big(\pi(\mu,\nu),\,\pi(\mu_N^{\varepsilon},\nu)\big)\Big]
\;\ge\;
\mathbb{E}\Big[ W_2\big(\mu,\,\mu_N^{\varepsilon}\big)\Big].
\]
Moreover, there exists $\mu$ supported on $[0,1]^m \times \{0\}^{d-m}$ such that for $m\ge 3$%
\nopagebreak
\[
\mathbb{E}\Big[ W_2\big(\mu,\,\mu_N^{\varepsilon}\big)\Big]
\;\ge\; \frac{1}{108m}\,N^{-1/m}-\varepsilon \mathbb{E}|Y|.\]
\end{theorem}

\begin{proof}[Proof of Theorem~\ref{thm:lower_bounds}]
We define the projection onto the first component by
\[
\operatorname{proj}_1:\cX\times\cY\to\cX,
\qquad
\operatorname{proj}_1(x,y)=x.
\]
Observe that $\operatorname{proj}_1$ is 1-Lipschitz continuous. \smallskip

For notational simplicity, set $\pi := \pi(\mu,\nu)$ and $\pi_N^\varepsilon := \pi(\mu_N^\varepsilon,\nu)$. From the left marginal constraints on these optimal plans, we have that
$(\operatorname{proj}_1)_\# \pi = \mu$ and
$(\operatorname{proj}_1)_\# \pi_N^\varepsilon = \mu_N^\varepsilon$.

Let $\gamma\in\Pi(\pi,\pi_N^\varepsilon)$ be any coupling between $\pi$ and
$\pi_N^\varepsilon$.
Then the pushforward $(\operatorname{proj}_1,\operatorname{proj}_1)_\#\gamma$ is a coupling between $(\operatorname{proj}_1)_\#\pi=\mu$ and
$(\operatorname{proj}_1)_\#\pi_N^\varepsilon=\mu_N^\varepsilon$. \smallskip

Using the $1$-Lipschitz property of $\operatorname{proj}_1$ and the definition of $W_2$, we obtain
\[
W_2(\mu,\mu_N^\varepsilon)^2
\le
\iint \|\operatorname{proj}_1(z)-\operatorname{proj}_1(z')\|^2 \, d\gamma(z,z')
\le
\iint |z-z'|^2 \cdot \, d\gamma(z,z').
\]
Taking the infimum over all $\gamma\in\Pi(\pi,\pi_N^\varepsilon)$ yields
\[
W_2(\mu,\mu_N^\varepsilon)
\le
W_2(\pi,\pi_N^\varepsilon).
\] Taking expectation implies the lower bound. \smallskip

Now assume $\mu=\text{Unif}([0,1]^m)\otimes\delta_0^{d-m}$.
Since $W_2\ge W_1$, we have
\[
W_2(\mu,\mu_N^\varepsilon)
\ge W_1(\mu,\mu_N^\varepsilon)
\ge W_1(\mu,\mu_N)-W_1(\mu_N,\mu_N^\varepsilon).
\]

To bound $W_1(\mu_N,\mu_N^\varepsilon)$, use the coupling
\[
\gamma:=\frac{1}{N}\sum_{i=1}^N \delta_{(X_i,X_i+\varepsilon Y_i)},
\]
which yields
\[
W_1(\mu_N,\mu_N^\varepsilon)
\le \varepsilon\Big(\frac{1}{N}\sum_{i=1}^N |Y_i|\Big).
\]
Taking expectation yields
\[
\mathbb{E}\,W_1(\mu_N,\mu_N^\varepsilon)
\le \varepsilon \mathbb{E}|Y|
\]

For the estimation term, $W_1(\mu,\mu_N)$, we use the dual formulation of $W_1$.
For any $\phi\in\text{Lip}_1(\mathbb{R}^d)$,
\[
\int_{\mathbb{R}^d}\phi\,d(\mu-\mu_N)
=
\int_{[0,1]^m}\phi(x,0)\,dx
-\frac{1}{N}\sum_{i=1}^N \phi(X_i).
\]
Writing $X_i=(U_i,0)$ with $U_i\sim\text{Unif}([0,1]^m)$ i.i.d. and defining
$\psi(u):=\phi(u,0)$, we have $\psi\in\text{Lip}_1(\mathbb{R}^m)$ and
\[
\int_{\mathbb{R}^d}\phi\,d(\mu-\mu_N)
=
\int_{\mathbb{R}^m}\psi\,d(\bar\mu-\bar\mu_N),
\quad
\bar\mu:=\text{Unif}([0,1]^m),\ 
\bar\mu_N:=\frac{1}{N}\sum_{i=1}^N\delta_{U_i}.
\]
Hence, from the dual formulation of the $W_1$ metric, $W_1(\mu,\mu_N)=W_1(\bar\mu,\bar\mu_N)$. Hence,
\citep[Theorem~2.14]{chewi2024statistical} implies
\[
\mathbb{E}\,W_1(\mu,\mu_N)
\ge \frac{1}{108m}\,N^{-1/m}.
\]
Combining the above inequalities completes the proof.

\end{proof}


\subsubsection{Lemma~\ref{lem:no_uniform_Lip}}\label{app:proof-no-uniform-lip}

\begin{lemma}[No $\varepsilon$-uniform local Lipschitz bound]\label{lem:no_uniform_Lip}
Assume quadratic cost and that the OT problem from $\mu$ to $\nu$ admits no optimal Monge map (with the support of $\nu$ being compact).
For $\varepsilon>0$ assume that $\mu_\varepsilon \ll \mathcal{L}^d$ and let $T_\varepsilon$ be the Monge map from $\mu_\varepsilon$ to $\nu$.
Then if $\mu_{\eps} \rightharpoonup \mu$, then there exists $R_0>0$ such that
\[
\sup_{0<\varepsilon\le 1}\ \|\nabla T_\varepsilon\|_{L^\infty(B_{R_0})}=+\infty.
\]
\end{lemma}

\begin{proof}
We argue by contradiction. Then for every $R > 0$, there exists an $L(R)<+\infty$ such that
\[
\sup_{0<\varepsilon\le 1}\ \|\nabla T_\varepsilon\|_{L^\infty(B_{R})}\leq L(R).
\]
Now as $T_{\varepsilon}$ transports $\mu_{\varepsilon}$ to $\nu$, we see that $T_{\varepsilon}$ is uniformly bounded because $\nu$ has bounded support.
Hence, by Arzel\`a--Ascoli, we have that along a subsequence, $T_{\varepsilon}  \rightarrow T_0$ locally uniformly.
This uniform convergence implies $(T_0)_{\#} \mu = \nu$.
Also by continuity of the Wasserstein-2 metric (see \citep[Theorem 1.51]{santambrogio}), we have that $T_0$ is an optimal transport map from $\mu$ to $\nu$, which is a contradiction.
\end{proof}

\subsection{A 1D toy example illustrating instability of Monge map gradients} 

Now we present a 1D toy example that illustrates the instability in Monge maps for small $\eps>0$, while the maps become more stable as $\eps$ gets larger.

\begin{example}[A stability proxy in 1D]\label{ex:1d_stability_proxy}

In the 1D toy case with $\mu=\delta_0$, Gaussian smoothing gives $\mu_\varepsilon=\mathcal N(0,\varepsilon^2)$.
Taking $\nu=\mathrm{Unif}([-1,1])$, the optimal map is
\[
T_\varepsilon(x)=2\Phi(x/\varepsilon)-1,
\]
where $\Phi$ is the standard normal CDF. Differentiating yields
\[
T_\varepsilon'(x)=\frac{2}{\varepsilon\sqrt{2\pi}}\exp\!\Big(-\frac{x^2}{2\varepsilon^2}\Big),
\quad\text{so}\quad
\|T_\varepsilon'\|_{L^\infty(\R)}=\frac{2}{\varepsilon\sqrt{2\pi}}.
\]
Thus, in this example, larger $\varepsilon$ produces a flatter (more stable) map, while $\varepsilon\downarrow 0$ produces an increasingly sharp map.
\end{example}

\subsubsection{Proof of Lemma~\ref{lem:theta_derivatives}}\label{app:proof-theta-derivatives}
Lemma~\ref{lem:theta_derivatives} is implied by the following lemma.

\begin{lemma}[Derivatives of the reduced semi-dual]\label{lem:theta_derivatives}
Let $c(x,y)=\tfrac12|x-y|^2$ and define
\[
\mathcal J(\theta):=\int V_\theta^{c}\,d\mu+\int V_\theta\,d\nu,
\qquad
V_\theta^{c}(x):=\inf_{y}\Big\{\tfrac12|x-y|^2-V_\theta(y)\Big\}.
\]
Assume the minimizer in the definition of $V_\theta^c$ exists and is unique and write it as $T_\theta(x)$; set
 and $G_\theta(x):=\nabla_x\nabla_\theta V_\theta(T_\theta(x))$.
If $V_\theta$ is $C^2$ and $A_\theta(x)$ is invertible $\mu$-a.e., then
\[
\nabla_\theta \mathcal J(\theta)=\int \nabla_\theta V_\theta\,d\nu-\int \nabla_\theta V_\theta(T_\theta(x))\,d\mu(x),
\]
and
\[
\nabla_{\theta\theta}^2 \mathcal J(\theta)
=\int \nabla_{\theta\theta}^2 V_\theta\,d\nu-\int \nabla_{\theta\theta}^2 V_\theta(T_\theta(x))\,d\mu(x)
-\int G_\theta(x)^\top \nabla_x T_{\theta}(x)G_\theta(x)\,d\mu(x).
\]
\end{lemma}

\begin{proof}
Let $T_\theta(x)$ be the unique minimizer in the $c$-transform. Then
\[
V_\theta^{c}(x)=\tfrac12|x-T_\theta(x)|^2 - V_\theta(T_\theta(x)).
\]
Observe that by differentiating the above equality we obtain
\[
\nabla_\theta V_\theta^{c}(x)=-\nabla_\theta V_\theta(T_\theta(x)),
\]
which gives the gradient formula after integrating.

For the Hessian, we first compute $A_{\theta}(x)$. By differentiating in $x$
the first order optimality condition
\begin{equation}
0=T_\theta(x)-x-\nabla V_\theta(T_\theta(x)). \label{first_order_optimality}
\end{equation} implies that
\[ \nabla_x T_{\theta} =  (I-\nabla_{xx}^2 V_{\theta}(T_{\theta}))^{-1}. \]
Differentiating this first order optimality condition in $\theta$ implies that
\[
\partial_\theta T_\theta(x)=(\nabla_x T_{\theta}) G_\theta(x).
\]
Thus $\partial_\theta T_\theta(x)=\nabla_x T_{\theta}(x)G_\theta(x)$. Plugging this into the chain rule for
$\partial_\theta\big(\nabla_\theta V_\theta(T_\theta(x))\big)$ yields
\[
\nabla_{\theta\theta}^2 \mathcal J(\theta)
=\int \nabla_{\theta\theta}^2 V_\theta\,d\nu-\int \nabla_{\theta\theta}^2 V_\theta(T_\theta)\,d\mu
-\int G_\theta^\top \nabla_x T_{\theta}  G_\theta\,d\mu.
\]
\end{proof}

\begin{remark} From the first order optimality condition \eqref{first_order_optimality} and Theorem~\ref{thm:recovery_characterization}, we see that if $V_{\theta}$ is an optimizer of \eqref{eq:opt_pair_V}, then $(V_{\theta},T_{\theta})$ is an optimal solution of \eqref{eq:snot_intro}. Hence, by Lemma~\ref{lem:full_recovery} under an absolute continuity assumptions on $\mu$, we have that $T_{\theta}$ is the unique Monge map from $\mu$ to $\nu$.
\end{remark}

\section{Implementation Details} \label{app:implementation_details}
This section provides the implementation details of the synthetic data experiments, including dataset constructions, model architectures, and the specific hyperparameters used for both the OTP and low-dimensional manifold benchmarks.

\subsection{Dataset Description}
For all datasets, let $x, y \in \mathbb{R}^d$ and define the manifold dimension $n = d/2$ for the OTP benchmark \citep{choi2025overcoming}. We decompose vectors as $x = (x_1, x_2)$ and $y = (y_1, y_2)$, where $x_1, y_1 \in \mathbb{R}^n$ and $x_2, y_2 \in \mathbb{R}^{d-n}$. Let $e_1 = (1, 0, \dots, 0) \in \mathbb{R}^n$.

\begin{itemize}
    \item \textbf{Perpendicular}:  $\mu$ is supported on $[-1, 1]^n \times \{0\}^{d-n}$ with $x_1 \sim \text{Unif}([-1, 1]^n)$. Similarly, $\nu$ is supported on $\{0\}^{d-n} \times [-1, 1]^n$ with $y_2 \sim \text{Unif}([-1, 1]^n)$.
    
    \item \textbf{One-to-Many}: $x \sim \mu$ as in Perpendicular. Target $y \sim \nu$ has $y_1 \sim \text{Unif}([-1, 1]^n)$, while $y_2$ is sampled from a discrete distribution $\text{Cat}(\{e_1, -e_1\}, \{0.5, 0.5\})$.
\end{itemize}

\subsection{Model Architecture and Training}
We maintain a consistent architecture across all synthetic experiments. Both the potential function $v_\phi$ and the transport map $T_\theta$ are parameterized as neural networks with a single hidden layer and ReLU activations. We use a hidden dimension of 256 for $d \in \{2, 4\}$ and 1024 for $d \ge 16$. We use the Adam optimizer with $(\beta_1, \beta_2) = (0, 0.9)$ and a learning rate of $10^{-4}$. Training runs for 20,000 iterations with a batch size of 128. To the optimization of the inner loop in the dual formulation, we perform $K_T = 20$ updates to $T_\theta$ for every single update to $v_\phi$.

\paragraph{Noise Annealing Implementation}
For smoothed models, we generate perturbed data $\hat{x} = x + \sigma z$, where $z \sim \mathcal{N}(\mathbf{0}, \mathbf{I})$.
\begin{itemize}
    \item \textbf{Baseline OTP}:  Follows a linear interpolation from $\sigma_{\text{max}} = 0.2$ to $\sigma_{\text{min}} = 0.05$. The noise level is updated every $P=2000$ iterations using a stepwise rule: $\sigma_k = (1-t) \sigma_{\text{max}} + t \sigma_{\text{min}}$, where $t = (\lfloor k/P \rfloor \cdot P + 1) / K$.
    \item Our Approach: Adopts the rate-optimal schedule $\varepsilon_{\text{eff}}(n) = \max(\varepsilon_{\text{stat}}(n), \varepsilon_{\text{min}})$ as defined in Eq. \ref{eq:estat}.
\end{itemize}

\subsection{Hyperparameters}
This section details the hyperparameters used for the numerical experiments. We denote $\tau$ as the cost intensity parameter in $c(x, y) = \tfrac{\tau}{2}\| x-y \|_{2}^{2}$ and $\lambda$ as the $R_1$ regularization weight \citep{which-converge}.

\subsubsection{Neural OT Plan accuracy evaluation}
For the low-dimensional manifold experiments ($d=256$), the original OTP baseline configurations ($\lambda=1$) were found to be unstable. To ensure a meaningful baseline, we increased the regularization weight to $\lambda=5$ for both OTM and OTP in that setting. Conversely, our approach remained stable with a single configuration across both benchmarks.

\begin{table}[h]
\centering
\caption{\textbf{Summary of Hyperparameters.} 
Here, $\tau$ denotes the cost intensity, $\lambda$ the $R_1$ regularization weight, and $\sigma_{\min}$ the terminal noise floor used in smoothing-based methods.}
\label{tab:hyperparams}
\small
\begin{tabular}{llcccc}
\toprule
\textbf{Benchmark} 
& \textbf{Model} 
& $\boldsymbol{\tau}$ 
& $\boldsymbol{\lambda}$ 
& $\boldsymbol{\sigma_{\min}}$ 
& \textbf{Remarks} \\
\midrule
\multirow{3}{*}{\shortstack[l]{Low-Dimensional Manifold \\ ($d = 256$)}} 
& OTM 
& 0.01 
& 1.0 
& -- 
& Matches OTP benchmark setting \\

& OTP 
& 0.01 
& 5.0 
& 0.05 
& Unstable for smaller $\lambda$ (e.g., $\lambda = 1$) \\

& Ours 
& 0.1 
& 3.0 
& 0.1 
& Consistent performance across dimensions \\

\midrule
\multirow{5}{*}{\shortstack[l]{High-Dimensional Benchmark \\ \citep{choi2025overcoming}}} 
& OTM 
& 1.0 
& 0.0 
& -- 
& Baseline configuration \\

\cmidrule{2-6}
& \multirow{2}{*}{OTP}
& 1.0 
& 0.0 
& 0.05 
& $d < 256$ as in \citep{choi2025overcoming} \\

& 
& 0.01 
& 1.0 
& 0.05 
& $d = 256$ as in \citep{choi2025overcoming} \\

\cmidrule{2-6}
& \multirow{2}{*}{Ours} 
& 1.0 
& 0.0 
& 0.1 / 0.05$^\dagger$ 
& Rate-optimal noise scheduling for $d<256$\\

& 
& 0.1 
& 3.0 
& 0.1 / 0.05$^\dagger$ 
& Rate-optimal noise scheduling for $d=256$ \\

\bottomrule
\multicolumn{6}{l}{\footnotesize $^\dagger$~$\sigma_{\min}=0.1$ for the Perpendicular setting and $0.05$ for the One-to-Many setting.}
\end{tabular}
\end{table}

\subsubsection{Validation of the terminal noise level}
The hyperparameters used to generate the results in Figure~\ref{fig:neural_map_fit} are summarized in Table~\ref{tab:hypers_map_fit}.

\begin{table}[h]
\centering
\caption{Summary of hyperparameters used for Figure~\ref{fig:neural_map_fit}}
\label{tab:hypers_map_fit}
\small
\begin{tabular}{lc}
\toprule
\textbf{Hyperparameter} & \textbf{Value} \\
\midrule
Ambient Dimension ($d$) & 10 \\
Manifold Dimension ($m$) & 10 \\
Sample size ($n$) & 20,000,000 \\
Hidden Dimension & 2048 \\
Cost Intensity ($\tau$) & 1.0 \\
Noise Level ($\varepsilon$) & $[0.005,\, 1]$ \\
Training Iterations & 5{,}000 \\
Map Regularization ($\lambda$) & 0.0 \\
$K_T$ & 1 \\
\bottomrule
\end{tabular}

\end{table}

\subsubsection{Regularization effect of Smoothing}
The hyperparameters used to generate the results in Figure~\ref{fig:mse_vs_n} are summarized in Table~\ref{tab:hypers_mse_vs_n}.

\begin{table}[h]
\centering
\caption{Summary of hyperparameters used for Figure~\ref{fig:mse_vs_n}}
\label{tab:hypers_mse_vs_n}
\small
\begin{tabular}{lc}
\toprule
\textbf{Hyperparameter} & \textbf{Value} \\
\midrule
Ambient Dimension ($d$) & 10 \\
Manifold Dimension ($m$) & 10 \\
Sample size ($n$) & 2,000,000 \\
Hidden Dimension & 2048 \\
Cost Intensity ($\tau$) & 1.0 \\
Training Iterations & 35,000 \\
Map Regularization ($\lambda$) & 0.0 \\
$K_T$ & 1 \\
\bottomrule
\end{tabular}

\end{table}

\section{Additional Experiments} \label{app:additional_experiment}
\subsection{Higher-dimensional Manifold Benchmarks} \label{app:standard_manifold}
We evaluate the numerical accuracy of the learned Neural OT models on synthetic datasets from \citep{choi2025overcoming} across various dimensions. Here, the intrinsic manifold dimension is set to half of the ambient dimension $m=d/2$ for $d \in \{2, 4, 16, 64, 256 \}$.  

The quantitative results are summarized in Table \ref{tab:accuracy_highdim_mfd}. In the Perpendicular Dataset, our rate-optimal strategy consistently achieves the highest accuracy. Specifically, in the high-dimensional $d=256$ case, our approach significantly outperforms the vanilla OTM and remains superior to the heuristic OTP. This confirms that grounding the noise schedule in intrinsic manifold rates improves map identifiability in complex ambient spaces.
In the One-to-Many Dataset, both our approach and the original OTP exhibit comparably strong performance, particularly for $d \ge 64$. Both smoothing-based methods successfully prevent the catastrophic "metric explosion" seen in the unregularized OTM (e.g., $d=16$), demonstrating that smoothing is essential for optimization stability in multi-modal transport tasks.

\begin{table*}[h]
    \centering
    \caption{\textbf{Quantitative comparison of numerical accuracy} on synthetic benchmarks from \citet{choi2025overcoming}. Models are evaluated on transport cost error $D_{cost} (\downarrow)$ and target distribution error $D_{target} (\downarrow)$. \textbf{Boldface} and \underline{underlined} values indicate the best and second-best performance, respectively. 
    }
    \label{tab:accuracy_highdim_mfd}
    \scalebox{0.8}{
    \begin{tabular}{c c c c c c c}
        \toprule
        \multirow{2}{*}{Ambient Dim ($d$)} & \multirow{2}{*}{Manifold Dim ($m$)} & \multirow{2}{*}{Model} & \multicolumn{2}{c}{Perpendicular} & \multicolumn{2}{c}{One-to-Many}  \\
        \cmidrule{4-5} \cmidrule{6-7}
        & & & $D_{cost}$ & $D_{target}$ & $D_{cost}$ & $D_{target}$ \\
        \midrule
        \multirow{3}{*}{$d=2$} & \multirow{3}{*}{$1$} & OTM & $0.129 \pm 0.113$ & $0.457 \pm 0.292$ & $0.141 \pm 0.109$ & $0.063 \pm 0.037$ \\
                                & & OTP & \underline{$0.034 \pm 0.021$} & $\underline{0.015} \pm 0.009$ & $\underline{0.018} \pm 0.011$ & $\underline{0.063} \pm 0.044$ \\
                                & & Ours & $\textbf{0.024} \pm 0.031$ & $\textbf{0.005} \pm 0.009$ & $\textbf{0.010} \pm 0.004$ & $\textbf{0.047} \pm 0.022$ \\
                                
        \midrule
        \multirow{3}{*}{$d=4$}  & \multirow{3}{*}{$2$} & OTM & $\underline{0.049} \pm 0.036$ & $0.029 \pm 0.012$ & $0.111 \pm 0.032$ & $0.060 \pm 0.040$ \\
                                & & OTP & $0.063 \pm 0.047$ & $\underline{0.026} \pm 0.024$ & $\textbf{0.017} \pm 0.007$ & $\underline{0.054} \pm 0.039$ \\
                                & & Ours & $\textbf{0.042} \pm 0.035$ & $\textbf{0.011} \pm 0.006$ & $\underline{0.018} \pm 0.008$ & $\textbf{0.045} \pm 0.025$ \\

        \midrule
        \multirow{3}{*}{$d=16$} & \multirow{3}{*}{$8$} & OTM & $0.265 \pm 0.143$ & $4.050 \pm 1.749$ & $34.454 \pm 34.442$ & $36.468 \pm 34.380$ \\
                                & & OTP  & $\underline{0.079} \pm 0.049$ & $\underline{0.618} \pm 0.020$ & $\textbf{0.097} \pm 0.026$ & $\textbf{0.693} \pm 0.059$ \\
                                & & Ours & $\textbf{0.066} \pm 0.032$ & $\textbf{0.616} \pm 0.018$ & $\underline{0.120} \pm 0.040$ & $\underline{0.763} \pm 0.073$ \\

        \midrule
        \multirow{3}{*}{$d=64$} & \multirow{3}{*}{$32$} & OTM & $\underline{2.175} \pm 1.115$ & $19.865 \pm 1.035$ & $0.287 \pm 0.154$ & $12.341 \pm 0.364$ \\
                                & & OTP & $2.435 \pm 3.031$ & $\underline{15.021} \pm 6.719$ & $\textbf{0.119} \pm 0.015$ & $\textbf{9.856} \pm 0.033$ \\
                                & & Ours  & $\textbf{0.401} \pm 0.237$ & $\textbf{10.367} \pm 0.482$ & $\underline{0.144} \pm 0.004$ & $\underline{9.883} \pm 0.036$ \\
                                
        \midrule
        \multirow{3}{*}{$d=256$} & \multirow{3}{*}{$128$} & OTM & $15.414 \pm 4.493$ & $80.841 \pm 2.630$ & $\textbf{0.211} \pm 0.049$ & $\underline{62.047} \pm 0.196$ \\
                                & & OTP & $\underline{6.198} \pm 3.376$ & $\underline{62.538} \pm 3.412$ & $\underline{0.363} \pm 0.028$ & $\textbf{61.377} \pm 0.083$ \\
                                & & Ours & $\textbf{5.609} \pm 0.274$ & $\textbf{60.656} \pm 0.237$ & $0.653 \pm 0.049$ & $61.595 \pm 0.090$ \\
                                
        \bottomrule
    \end{tabular}}
\end{table*}

\begin{figure}[h]
    \centering
    \includegraphics[width=0.4\textwidth]{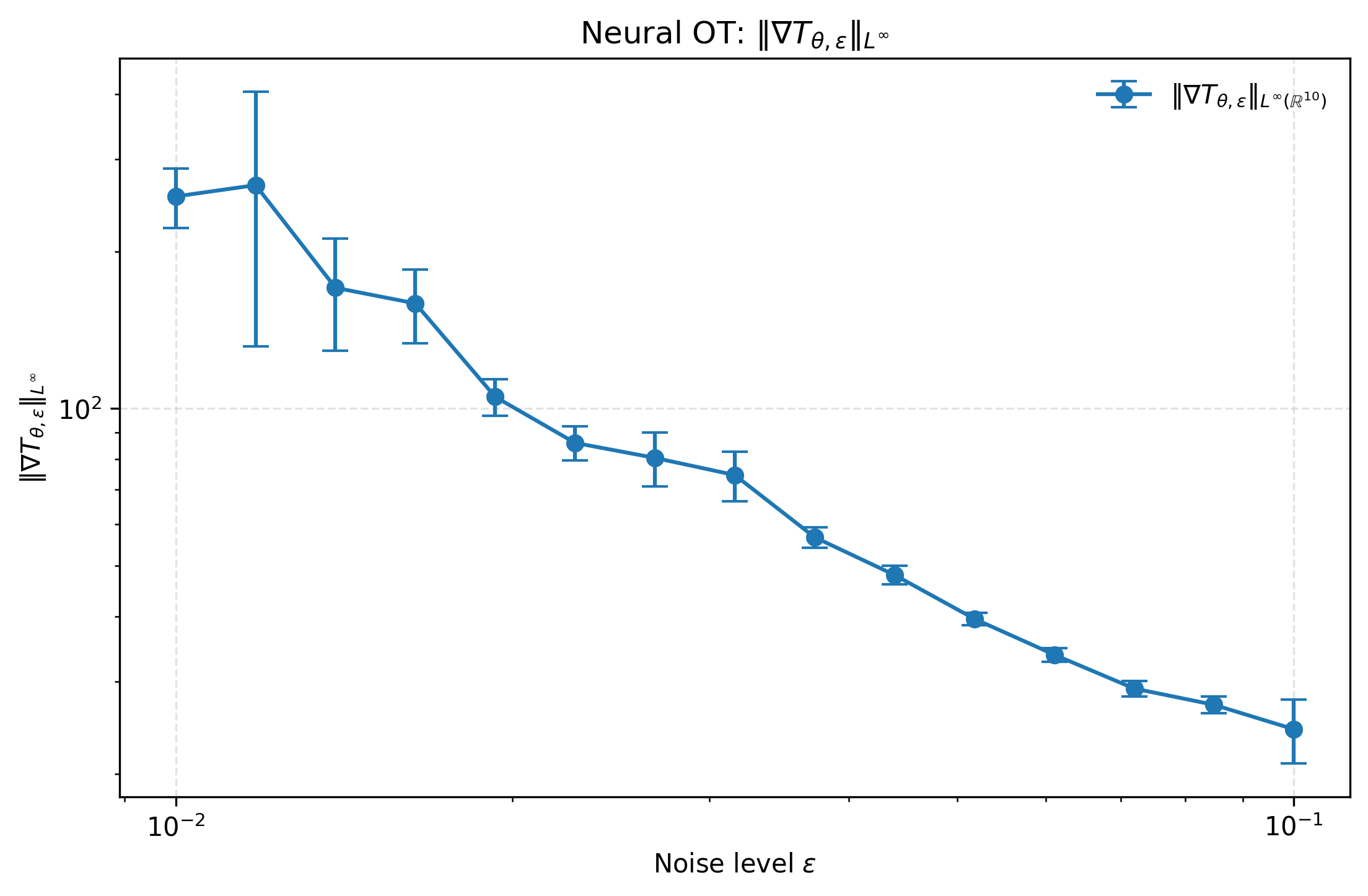}
    \caption{\textbf{Learned map gradient norm vs. $\varepsilon$.} As noise vanishes, $\|\nabla T_{\theta,\varepsilon}\|_{L^\infty}$ diverges, reflecting the singularity of the exact OT map. This gradient blow-up shows the ill-conditioning observed in training.} 
    \label{fig:map_gradient_blow_up}
    \vspace{-8pt}
\end{figure}

\subsection{Regularization effect of Smoothing}
We empirically validate the theoretical predictions established in Theorem \ref{lem:theta_derivatives1} and Lemma \ref{lem:no_uniform_Lip} regarding the behavior of Neural OT maps under vanishing smoothing. By considering the specific case where $\mu=\delta_0$ and $\nu=\mathcal{N}(0,I_{10})$, the smoothed source measure becomes $\mu_{\eps} = \mathcal{N}(0,\eps^2 I_{10})$. In this regime, the optimal Monge map is analytically given by $T_\varepsilon(x)=x/\varepsilon$, which possesses a Jacobian $\nabla T_\varepsilon=\varepsilon^{-1}I_{10}$ that diverges as $\varepsilon \to 0$.

The results in Figure \ref{fig:map_gradient_blow_up} confirm that the learned Neural OT map inherits this small-$\varepsilon$ gradient blow-up. As the noise level $\varepsilon$ vanishes, the supremum norm of the gradient $\|\nabla T_{\theta,\varepsilon}\|_{L^\infty}$ increases rapidly, mirroring the $\varepsilon^{-1}$ singularity of the exact transport map. Because Theorem \ref{lem:theta_derivatives1} explicitly links the Hessian of the training objective $\nabla^2_{\theta\theta}\mathcal{J}(\theta)$ to the Jacobian $\nabla T_\theta$, this gradient divergence provides a clear numerical mechanism for the ill-conditioning observed in training.

\section{Related Works} \label{app:related_works}
Neural Optimal Transport methods aim to learn the optimal transport maps between the source and target distributions using neural networks. Many approaches proposed methods based on the semi-dual formulation of the optimal transport problem \cite{otm, fanscalable, uotm, otmICNN, fanTMLR}. These approaches have been successfully applied to generative modeling \citep{otm, fanTMLR, uotm, sjko}, point cloud completion \citep{uot-upc}, and inverse problems \citep{gazdieva2025optimal}. 

Despite these successes, Semi-dual Neural Optimal Transport (SNOT) suffers from a fundamental limitation known as the \textbf{spurious solution problem}. The spurious solution problem refers to a failure mode of SNOT, where the solution to the max-min objective $\mathcal{L}(V, T)$ \eqref{eq:snot_intro} fails to recover the correct optimal transport map. Although the ground-truth Kantorovich potential and transport map $(V^{\star}, T^{\star})$ are a solution to the max-min formulation, the converse does not generally hold \citep{otm, fanTMLR}. In particular, 
even when the max-min objecive is solved exactly, the recovered map $T_{rec}$ may fail to push the source measure $\mu$ to the target $\nu$, leading to incorrect transport maps.

\citet{fanTMLR} showed that, when $\mu$ is atomless, a stronger condition of optimality, i.e., the saddle point solution $(V_{sad}, T_{sad})$ of $\mathcal{L}_{V, T}$, is equivalent to the Kantorovich optimal potential and optimal transport map. Specifically, a saddle point satisfies
\begin{equation} \label{eq:saddle_def}
        V_{sad} \in \arg\max_{V} \mathcal{L} (V, T_{sad}), \qquad
        T_{sad} \in \arg\min_{T} \mathcal{L} (V_{sad}, T).    
\end{equation} 
However, this recovery is not guaranteed for the generic max-min solution of $\mathcal{L}_{V, T}$.

More recently, \citet{choi2025overcoming} established a sufficient condition under which the max-min solution recovers the optimal transport map. Their condition requires that the source measure $\mu$ does not assign positive mass to measurable sets of Hausdorff dimension at most $d-1$, where $d$ denotes the ambient dimension. To enforce this condition in practice, they proposed smoothing regularization of the source measure and analyzed the resulting convergence of the learned optimal transport plan $\pi^{\star}$.

In this paper, we provide a sharper and more comprehensive theoretical analysis of the spurious solution problem within the SNOT framework. First, while \citet{choi2025overcoming} derived a sufficient condition to prevent spurious solutions, we characterize the degree to which the max-min solution remains accurate even when this condition is not satisfied (Theorems~\ref{thm:recovery_characterization} \& \ref{thm:tangential_recovery}). Second, while prior work established only weak-convergence results for the optimal transport plan $\pi^{\star}$, we prove map-level recovery guarantees for the transport map $T^{\star}$ (Theorem~\ref{thm:map_level_convg}). Third, we derive quantitative convergence rates for $\pi^{\star}$ (Theorem~\ref{thm:semi_discrete}), which are significantly stronger than previous qualitative results. Finally, based on these quantitative rates, we propose a principled noise annealing schedule $\varepsilon_{\textnormal{stat}}(N)$.

\end{document}